\documentclass{article}
\usepackage{iclr2027_conference,times}
\iclrfinalcopy

\usepackage[utf8]{inputenc}
\usepackage[T1]{fontenc}
\usepackage{hyperref}
\usepackage{url}
\usepackage{booktabs}
\usepackage{amsfonts}
\usepackage{amsmath,amssymb,amsthm}
\usepackage{nicefrac}
\usepackage{microtype}
\usepackage[table]{xcolor}
\usepackage{xspace}
\usepackage{graphicx}
\usepackage{subcaption}
\usepackage{multirow}
\usepackage{tikz}
\usepackage{pgfplots}
\pgfplotsset{compat=1.18}
\usetikzlibrary{arrows.meta,positioning,fit,backgrounds,decorations.pathreplacing}

\newtheorem{theorem}{Theorem}
\newtheorem{proposition}{Proposition}
\newtheorem{lemma}{Lemma}
\newtheorem{corollary}{Corollary}
\newtheorem{definition}{Definition}
\newtheorem{remark}{Remark}

\newcommand{\method}{Hyper-Fold\xspace}
\newcommand{\methoddeep}{Hyper-Fold-Deep\xspace}
\newcommand{\methodpocket}{Hyper-Fold-Pocket\xspace}

\newcommand{\vx}{\mathbf{x}}
\newcommand{\vdelta}{\boldsymbol{\delta}}
\newcommand{\vh}{\mathbf{h}}
\newcommand{\mW}{\mathbf{W}}
\newcommand{\mU}{\mathbf{U}}
\newcommand{\mV}{\mathbf{V}}
\newcommand{\mB}{\mathbf{B}}

\newcommand{\mE}{\mathbf{E}}
\newcommand{\mP}{\mathbf{P}}
\newcommand{\mX}{\mathbf{X}}
\newcommand{\mK}{\mathbf{K}}
\newcommand{\mM}{\mathbf{M}}
\newcommand{\mQ}{\mathbf{Q}}

\title{Hyper-Fold: Exploring the Expressive Limit of\\Sequence-Geometry Learning for Proteins via Hypergraph Modeling}

\author{Yifan Feng$^{1}$ \quad Guanjie Cheng$^{2}$ \quad Shihui Ying$^{3}$ \quad Shaoyi Du$^{4}$ \quad Yue Gao$^{1}$\\
$^{1}$\{School of Software, BNRist, THUIBCS\}, Tsinghua University, Beijing 100084, China\\
$^{2}$School of Software Technology, Zhejiang University, Ningbo 315100, China\\
$^{3}$Shanghai Institute of Applied Mathematics and Mechanics,\\
\quad School of Mechanics and Engineering Science, Shanghai University, Shanghai 200072, China\\
$^{4}$State Key Laboratory of Human-Machine Hybrid Augmented Intelligence,\\
\quad Institute of Artificial Intelligence and Robotics, Xi'an Jiaotong University, Xi'an 710049, China\\
\texttt{evanfeng97@gmail.com; chengguanjie@zju.edu.cn; shying@shu.edu.cn;}\\
\texttt{dushaoyi@xjtu.edu.cn; gaoyue@tsinghua.edu.cn}
}

\begin{document}

\maketitle

\begin{abstract}
Protein structure modeling rests on a single computational primitive: the interaction
between what a residue \emph{is} (sequence content) and \emph{where it sits}
(three-dimensional geometry). What is the expressive limit of this layer class? We show
that the complete bilinear operator
over content--geometry outer products---the sufficient statistic of all second-order
interactions---is the expressive ceiling, while the additive message passing of
mainstream geometric GNNs is provably blind to content--geometry binding. We then
introduce \method, a rank-$K$ separable convolutional backbone approaching this ceiling
at message-passing cost: each radius neighborhood is organized into a sequence hyperedge
and a contact hyperedge, modulated by an edge-conditioned matrix-valued operator
factorized into $K$ learned basis operators with geometry-generated coefficients.
Across enzyme function prediction, fold classification, and ligand binding site
detection, \method and its hierarchical variant \methoddeep achieve the best results
among protein-specific structure encoders; \methodpocket, an anchored set-prediction
head, surpasses UniSite-3D on UniSite-DS and two zero-shot benchmarks with no sequence
language model features, $68\times$ fewer parameters, and $4.8\times$ lower
latency---suggesting that a sufficiently expressive 3D backbone recovers information
that fusion architectures previously borrowed from evolution-scale pretraining.
\end{abstract}

\section{Introduction}
\label{sec:intro}

Protein function is determined by three-dimensional structure, and structure modeling is
dominated by a single computational primitive: the \emph{interaction} between what
a residue \emph{is} (its sequence/content features) and \emph{where it sits} relative to its
neighbors (three-dimensional geometry). Message-passing (MP) architectures
\citep{jing2021gvp,zhang2023gearnet,fan2023cdconv} implement this primitive by propagating
geometry-aware messages over residue graphs and remain the most effective structure encoders
at modest scale \citep{jamasb2024benchmark}; protein language models and hybrid variants
\citep{lin2023esm2,su2024saprot,hayes2025esm3} fold structure into large-scale sequence
pretraining. Yet beneath this diversity, no one has systematically answered a basic question:
\emph{what is the expressive limit of sequence--geometry interaction layers, and where do
existing models sit relative to that limit?} We answer this question: any second-order
statistic of a neighbor's content and geometry is a linear functional of their outer
product $\vx_j \otimes \phi(\vdelta_{ij})$, so the \emph{complete bilinear operator}
over this outer product is the sufficient statistic---and hence the expressive
limit---of sequence--geometry interaction (Theorem~\ref{thm:ceiling}). The result is an
expressivity ladder (Figure~\ref{fig:teaser}(a)).

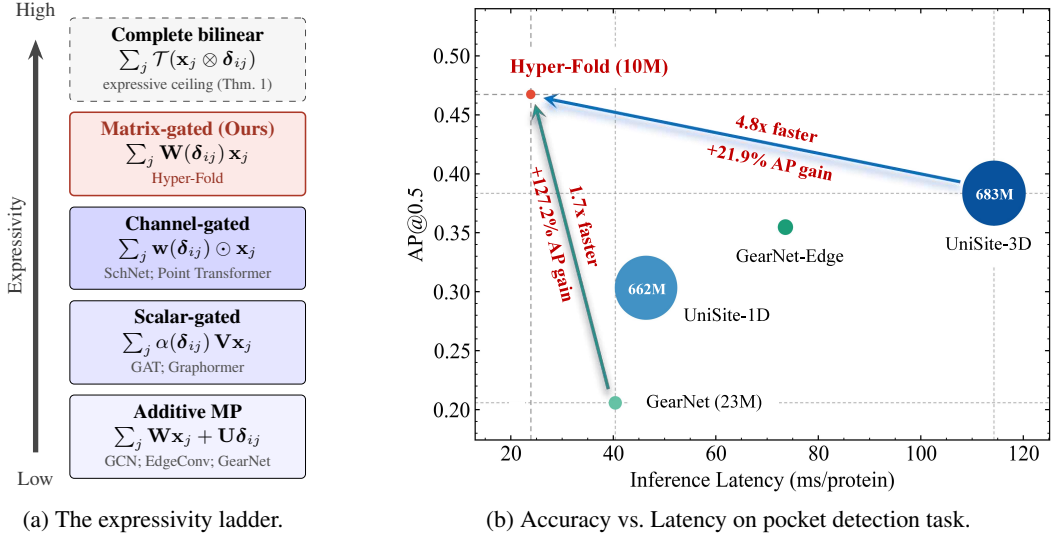
\begin{figure}[t]
\centering
\begin{subfigure}[t]{0.30\linewidth}
\centering
\resizebox{\linewidth}{!}{
\definecolor{npgcoral}{HTML}{E64B35}
\begin{tikzpicture}[
    font=\small,
    rung/.style={draw, rounded corners=2pt, align=center, minimum width=3.75cm, minimum height=1.0cm, inner sep=4pt},
  ]
  \node[rung, fill=blue!5]  (add) at (0,0)   {\textbf{Additive MP}\\[1pt] $\sum_j \mW\vx_j + \mU\vdelta_{ij}$\\[1pt] {\scriptsize\textcolor{gray!65!black}{GCN; EdgeConv; GearNet}}};
  \node[rung, fill=blue!10]  (sca) at (0,1.5) {\textbf{Scalar-gated}\\[1pt] $\sum_j \alpha(\vdelta_{ij})\,\mV\vx_j$\\[1pt] {\scriptsize\textcolor{gray!65!black}{GAT; Graphormer}}};
  \node[rung, fill=blue!16] (cha) at (0,3.0) {\textbf{Channel-gated}\\[1pt] $\sum_j \mathbf{w}(\vdelta_{ij})\odot\vx_j$\\[1pt] {\scriptsize\textcolor{gray!65!black}{SchNet; Point Transformer}}};
  \node[rung, fill=npgcoral!12, draw=npgcoral!75!black, line width=0.7pt] (mat) at (0,4.5)
    {\textbf{\textcolor{npgcoral!70!black}{Matrix-gated (Ours)}}\\[1pt] $\sum_j \mW(\vdelta_{ij})\,\vx_j$\\[1pt] {\scriptsize\textcolor{npgcoral!65!black}{\method}}};
  \node[rung, fill=gray!7, draw=gray!60!black, dashed] (bil) at (0,6.0)
    {\textbf{Complete bilinear}\\[1pt] $\sum_j \mathcal{T}(\vx_j \otimes \vdelta_{ij})$\\[1pt] {\scriptsize\textcolor{gray!60!black}{expressive ceiling (Thm.~\ref{thm:ceiling})}}};

  \draw[-{Stealth[length=3.2mm,width=2.4mm]}, line width=2.4pt, gray!55!black]
    (-2.42,-0.25) -- (-2.42,6.32);
  \node[rotate=90, gray!45!black, font=\small] at (-2.72,3.05) {Expressivity};
  \node[gray!45!black, font=\small, anchor=north] at (-2.42,-0.44) {Low};
  \node[gray!45!black, font=\small, anchor=south] at (-2.42,6.50) {High};
\end{tikzpicture}}
\caption{The expressivity ladder.}
\label{fig:teaser-ladder}
\end{subfigure}\hfill
\begin{subfigure}[t]{0.61\linewidth}
\centering
\includegraphics[width=\linewidth]{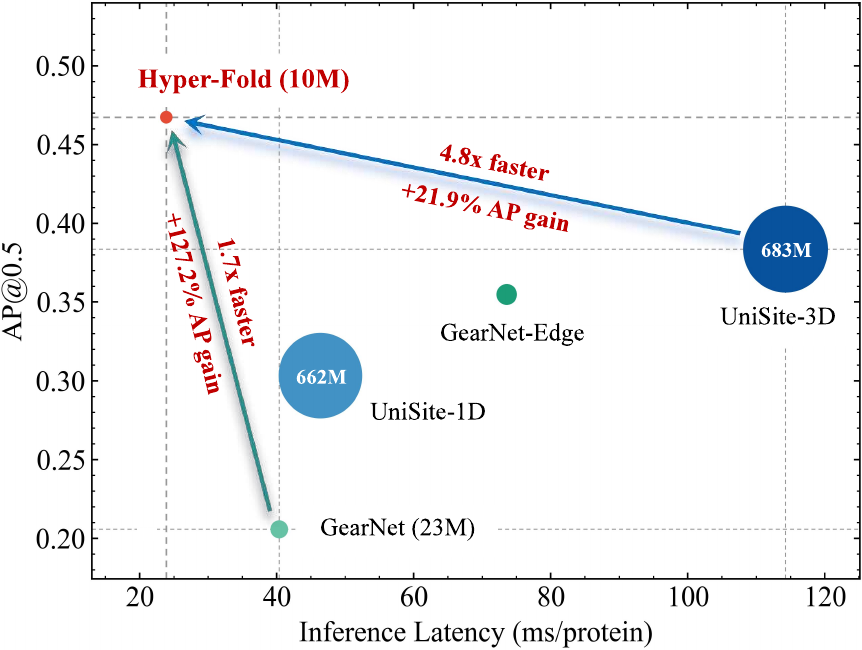}
\caption{Accuracy vs.\ Latency on pocket detection task.}
\label{fig:teaser-lat}
\end{subfigure}
\caption{Overview. \textbf{(a)} The expressivity ladder of sequence--geometry
layers: additive message passing occupies the lowest rung, while \method's rank-$K$
matrix gating approaches the complete-bilinear ceiling at message-passing cost.
\textbf{(b)} Pocket AP@0.5 on UniSite-DS vs.\ inference latency; bubble area scales
with parameter count.}
\label{fig:teaser}
\end{figure}

Situating existing models on this ladder exposes a structural blind spot. The additive message
used by mainstream geometric GNNs, $\mW\vx_j + \mU\vdelta_{ij}$, factorizes into a content term
and a geometry term; consequently it is provably invariant to \emph{binding swaps}---pairs of
structures whose node features and edge features are identical as multisets but whose
content--geometry pairings differ (Proposition~\ref{prop:binding}). Such binding information is
physically real (which residue type occupies which geometric position is precisely what
distinguishes a catalytic site from a geometrically identical decoy), yet it is invisible to the
lowest rung of the ladder. This gap belongs to the layer class: GearNet explicitly
motivated its additive message by the prohibitive cost of the
alternative \citep{zhang2023gearnet}---cost, not expressivity, was the bottleneck.

We therefore ask whether the ceiling can be \emph{affordably} approached. Our answer is
\method, a rank-$K$ separable convolutional backbone. Inspired by hypergraph modeling
\citep{feng2019hgnn,yadati2019hypergcn}, each layer organizes a radius neighborhood into
a \emph{sequence hyperedge} and a \emph{contact hyperedge}---sequence adjacency along
the chain versus spatial contact in the folded state---and applies to both an
edge-conditioned \emph{matrix-valued} operator $\mW(\vdelta_{ij})\,\vx_j$, factorized
as $\mW(\vdelta) = \sum_{c=1}^{K} g_c(\vdelta)\,\mB_c$ with geometry-generated
coefficients and learned basis operators. The operator spectrum decays rapidly:
$K{=}8$ captures 90--98\% of the learned full-operator energy and accuracy saturates
beyond it (Proposition~\ref{prop:rankk}), so the ceiling is approachable at
message-passing cost. The same operator yields two backbones: \method at full residue
resolution and \methoddeep for protein-level prediction.

We validate \method on three tasks that span the uses of protein structure encoders:
enzyme commission (EC) number prediction, fold classification, and ligand binding site
(pocket) detection. \method improves over GearNet-Edge on EC at every sequence-identity cutoff
(e.g., 0.864 vs.\ 0.810 Fmax at the 95\% cutoff) and on all three fold-classification splits.
For pocket detection we attach \methodpocket, an anchored set-prediction head on the
residue-level backbone: it reaches 0.617 AP@0.3 on UniSite-DS (UniSite-3D: 0.560) and
generalizes zero-shot to HOLO4K and COACH420, exceeding UniSite-3D on all six reported
metrics, \emph{without any sequence language model features}. This suggests that a sufficiently
expressive 3D backbone recovers information that fusion architectures previously borrowed
from evolution-scale pretraining.

Our contributions are: (i) an expressivity ladder for sequence--geometry layers, with
the complete bilinear operator characterized as its ceiling (Thm.~\ref{thm:ceiling}) and
additive message passing proved blind to content--geometry binding
(Prop.~\ref{prop:binding}); (ii) \method, a rank-$K$ separable convolutional backbone
that approaches this ceiling at message-passing cost; and (iii) experiments on three
tasks, including state-of-the-art pocket detection and zero-shot generalization without
sequence-model features, at essentially no extra efficiency cost.

\paragraph{Related work.} Pocket detection has evolved from geometry tools
\citep{leguilloux2009fpocket,krivak2018p2rank} to set-prediction formulations
\citep{fan2025unisite}, the formulation \methodpocket adopts with a purely structural
encoder. An extended discussion of structure encoders, sequence models, graph
expressivity, and hypergraph modeling is deferred to Appendix~\ref{app:related}.

\section{The Expressive Ladder of Sequence--Geometry Layers}
\label{sec:theory}

A protein of $L$ residues is a graph $\mathcal{G}=(\mathcal{V},\mathcal{E})$: node $i$
carries a \emph{content} feature $\vx_i \in \mathbb{R}^{C}$ (learned residue embeddings),
and each ordered edge $(i,j)\in\mathcal{E}$ within a radius cutoff carries a
\emph{geometry} feature $\vdelta_{ij}\in\mathbb{R}^{d_g}$---a 7-dimensional SE(3)-invariant
descriptor of the relative placement of residues $i$ and $j$ (distance, direction, and
orientation terms)---so all layers below are SE(3)-invariant by construction. A structure
encoder stacks layers that update node states $\vh_i$ by aggregating over the neighborhood
$\mathcal{N}(i)$. We classify such aggregation layers by the \emph{algebraic form}
of the interaction between a neighbor's content $\vx_j$ and the edge geometry
$\vdelta_{ij}$.

\begin{definition}[Five classes of sequence--geometry interaction layers]
\label{def:classes}
With learnable parameters suppressed:
\begin{align}
\text{(additive MP)}\quad & \vh_i = \textstyle\sum_{j\in\mathcal{N}(i)} \big(\mW\vx_j + \mU\vdelta_{ij}\big), \label{eq:additive}\\
\text{(scalar-gated)}\quad & \vh_i = \textstyle\sum_{j\in\mathcal{N}(i)} \alpha(\vdelta_{ij})\,\mV\vx_j, \quad \alpha(\cdot)\in\mathbb{R}, \label{eq:scalar}\\
\text{(channel-gated)}\quad & \vh_i = \textstyle\sum_{j\in\mathcal{N}(i)} \mathbf{w}(\vdelta_{ij})\odot\vx_j, \quad \mathbf{w}(\cdot)\in\mathbb{R}^{C}, \label{eq:channel}\\
\text{(matrix-gated)}\quad & \vh_i = \textstyle\sum_{j\in\mathcal{N}(i)} \mW(\vdelta_{ij})\,\vx_j, \label{eq:matrix}\\
\text{(complete bilinear)}\quad & \vh_i = \textstyle\sum_{j\in\mathcal{N}(i)} \mathcal{T}\big(\vx_j \otimes \phi(\vdelta_{ij})\big), \label{eq:bilinear}
\end{align}
where $\alpha$ is a scalar gate, $\mathbf{w}$ a per-channel (diagonal) gate, $\mW(\cdot)$ a
\emph{matrix-valued} kernel, $\phi$ a geometry feature map (the identity, or any richer
expansion of $\vdelta$), and $\mathcal{T}$ an arbitrary linear map on the outer product.
\end{definition}

Class~\eqref{eq:additive} is the message form of GCN, EdgeConv, and GearNet(-Edge)
\citep{kipf2017gcn,wang2019dgcnn,zhang2023gearnet}; class~\eqref{eq:scalar} covers
attention-style aggregation (GAT, Graphormer)
\citep{velickovic2018gat,ying2021graphormer,vaswani2017attention};
class~\eqref{eq:channel} covers continuous-filter convolutions and vector attention (SchNet,
Point Transformer) \citep{schutt2017schnet,zhao2021pointtransformer};
classes~\eqref{eq:matrix}--\eqref{eq:bilinear} are progressively richer.

\begin{proposition}[Containment]
\label{prop:containment}
Additive MP, scalar-gated, and channel-gated layers are degenerate special cases of
matrix-gated layers (after augmenting $\vx$ with a constant channel for the bias term):
channel gating is $\mW(\vdelta)$ restricted to diagonal matrices, scalar gating further ties
the diagonal entries into a single factor $\alpha(\vdelta)\mV$, and additive MP is recovered by
a rank-1, geometry-affine operator. Complete bilinear layers contain matrix-gated layers:
expanding the kernel in a basis of the matrix space,
$\mW(\vdelta)\vx=\sum_{c}g_c(\vdelta)\,\mB_c\vx$, the map is linear in
$\vx\otimes\phi(\vdelta)$ whenever the coordinates of $\phi$ represent the coefficient
functions $g_c$; a universal $\phi$ represents every continuous kernel
(proofs in Appendix~\ref{app:proofs}).
\end{proposition}

\begin{proposition}[Additive layers are blind to content--geometry binding]
\label{prop:binding}
Fix a center node $i$ with two neighbors and consider two configurations that differ only by a
\emph{binding swap}: configuration A pairs $(\vx_1,\vdelta_1),(\vx_2,\vdelta_2)$ while
configuration B pairs $(\vx_1,\vdelta_2),(\vx_2,\vdelta_1)$, with $\vx_1\neq\vx_2$ and
$\vdelta_1\neq\vdelta_2$. Then every additive MP layer \eqref{eq:additive} produces identical
$\vh_i$ for A and B, because its output factorizes as
$\sum_j \mW\vx_j + \sum_j \mU\vdelta_{ij}$. In contrast, there exist matrix-gated operators
\eqref{eq:matrix} that distinguish A from B, e.g.\ whenever
$\mW(\vdelta_1)(\vx_1-\vx_2) \neq \mW(\vdelta_2)(\vx_1-\vx_2)$ (proof in
Appendix~\ref{app:proofs}).
\end{proposition}

Proposition~\ref{prop:binding} is not a corner case: additive layers cannot represent
\emph{any} function of how content and geometry are paired within a neighborhood---which
residue type sits at which geometric position. Gated layers distinguish binding swaps with
increasing channel structure---one direction shared by all channels (scalar), $C$
directions that never mix (channel), full cross-channel interaction (matrix)---so the
ladder is a hierarchy of how richly content and geometry may be bound.

\begin{theorem}[The bilinear ceiling]
\label{thm:ceiling}
Every layer whose message is a second-order (bilinear) function of a neighbor's
content $\vx_j$ and geometry features $\phi(\vdelta_{ij})$ factors through the outer
product $\vx_j \otimes \phi(\vdelta_{ij})$; conversely, class~\eqref{eq:bilinear} realizes
all such functions. Hence the complete bilinear operator is the sufficient statistic of
second-order sequence--geometry interaction and constitutes the expressive ceiling of
sequence--geometry layers.
\end{theorem}

The ceiling is not directly usable: $\mathcal{T}$ has $O(C^2 d_\phi)$ degrees of freedom
per edge type. The question becomes how much structure the operator actually has.

\begin{proposition}[Rank-$K$ sufficiency, empirical]
\label{prop:rankk}
Writing the matrix-valued kernel in a learned basis,
$\mW(\vdelta) = \sum_{c=1}^{K} g_c(\vdelta)\,\mB_c$,
we find that the spectrum of the learned full operators decays rapidly: measured on the
trained $K{=}16$ model, the top-$8$ eigendirections capture 89.9--98.2\% of the operator
energy in every layer (Appendix~\ref{app:rankk}), and enlarging $K$ from $8$ to $16$
leaves pocket AP essentially unchanged. Rank-$K$ matrix gating therefore approaches the ceiling
at a small constant cost. A formal approximation bound in terms of the operator spectrum
tail is given in Appendix~\ref{app:proofs} (Proposition~\ref{prop:spectral}).
\end{proposition}

\section{\method: Matrix-Gated Convolution over Sequence and Contact Hyperedges}
\label{sec:method}

\method consists of structure backbones built from a single convolutional operator,
Fold-Conv (\S\ref{sec:method:layer}, Figure~\ref{fig:arch}), instantiated in two variants
(\S\ref{sec:method:backbone}): \method preserves the sequence length and produces
residue-level features, while \methoddeep interleaves residue pooling and produces
protein-level features. For pocket detection we attach \methodpocket, an anchored
set-prediction head on the residue-level backbone (\S\ref{sec:method:head},
Figure~\ref{fig:head}).

\subsection{Sequence and contact hyperedges}
\label{sec:method:hyper}

A residue relates to its neighborhood through two qualitatively different channels:
\emph{sequence adjacency} along the backbone chain and \emph{spatial contact} in the
folded state. Inspired by hypergraph modeling
\citep{feng2019hgnn,yadati2019hypergcn}---where a hyperedge connects a center vertex to a
\emph{set} of related vertices---we organize the radius neighborhood
$\mathcal{N}(i)=\{j:\|\mathbf{p}_i-\mathbf{p}_j\|<r\}$ of each residue $i$ into two
hyperedges (Figure~\ref{fig:arch}(a)): a \emph{sequence hyperedge} over the sequence
neighbors $\mathcal{N}_{\mathrm{seq}}(i)=\{j\in\mathcal{N}(i):|s_i-s_j|\le \ell/2\}$ and
a \emph{contact hyperedge} over the remaining spatial-contact neighbors
$\mathcal{N}_{\mathrm{con}}(i)=\mathcal{N}(i)\setminus\mathcal{N}_{\mathrm{seq}}(i)$,
where $s_i$ is the sequence index and $\ell$ the sequence-window width. We deliberately
do \emph{not} adopt hypergraph propagation rules; the hyperedge view only organizes the
neighborhood, and both hyperedges are processed by the same matrix-gated convolution
below, whose kernel network treats the two hyperedge types with separate embedding
weights (\S\ref{sec:method:layer}).

The backbone input is a learned residue embedding
$\mE_{\mathrm{res}}\in\mathbb{R}^{L\times 16}$ of the amino-acid type together with the
C$_\alpha$ position matrix $\mP_{\mathrm{res}}\in\mathbb{R}^{L\times 3}$; each residue
additionally carries a local orientation frame $\mathbf{O}_i\in\mathbb{R}^{3\times3}$
built from the neighboring C$_\alpha$ trace.

\begin{figure}[!h]
\centering
\includegraphics[width=1\linewidth]{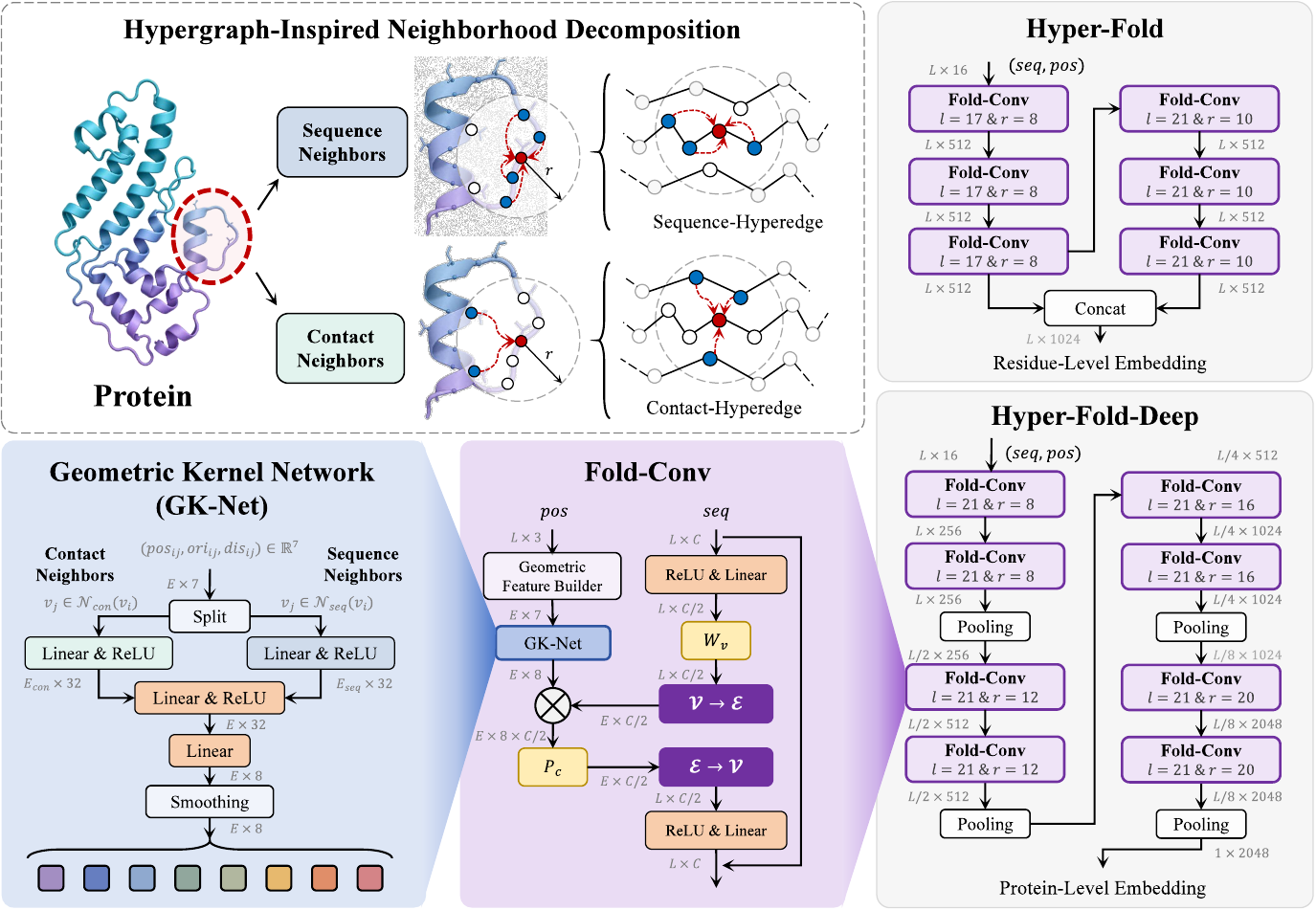}
\caption{\method backbones. Each residue anchors a \emph{sequence hyperedge} and a
\emph{contact hyperedge}; GK-Net maps each edge's geometry feature to rank-$K$ kernel
coefficients, modulating neighbor content in a vertex-to-edge / edge-to-vertex
convolution. \method produces residue-level embeddings; \methoddeep interleaves
residue pooling along a widening channel pyramid. Details
in Appendix~\ref{app:arch}.}
\label{fig:arch}
\end{figure}

\subsection{The Fold-Conv layer}
\label{sec:method:layer}

\paragraph{Geometric feature builder.}
A layer of radius $r$ connects every ordered pair within $\mathcal{N}(i)$ (plus
self-loops) and equips each edge with a 7-dimensional SE(3)-invariant geometry feature
and a clamped sequence displacement:
\begin{equation}
\small
\label{eq:edgefeat}
\vdelta_{ij}=\Big[\,\mathbf{O}_j^{\!\top}\frac{\mathbf{p}_i-\mathbf{p}_j}{\|\mathbf{p}_i-\mathbf{p}_j\|}
\,;\, \textstyle\sum_{k=1}^{3}\mathbf{O}_j^{(k)}\!\circ\mathbf{O}_i^{(k)}
\,;\, \|\mathbf{p}_i-\mathbf{p}_j\|\,\Big]\in\mathbb{R}^{7},
\enspace
\Delta s_{ij}=\mathrm{clamp}(s_i-s_j,\,-\ell/2,\,\ell/2) ,
\end{equation}
where $\mathbf{O}^{(k)}$ denotes the $k$-th frame axis and $\circ$ element-wise
multiplication. We abbreviate the three blocks of $\vdelta_{ij}$ as
$(pos_{ij},ori_{ij},dis_{ij})$: the relative position in $j$'s local frame, the relative
orientation, and the distance (Figure~\ref{fig:arch}(b)). Spatial contact and sequence
displacement thus enter the \emph{same} edge object, and the kernel below is conditioned
jointly on both---a convolution over fused sequence--contact structure rather than two
separate aggregation streams.

\paragraph{Geometric kernel network (GK-Net).}
The kernel coefficients of each edge are generated by a small geometric kernel network
(Figure~\ref{fig:arch}(b)). Its first-layer weights are indexed by a discretized sequence
bucket, which is what separates the two hyperedge types: sequence neighbors occupy the
$\ell$ offset buckets, while contact neighbors---whose sequence displacement saturates
the clamp---share the boundary bucket:
\begin{equation}
\label{eq:weightnet}
\mathbf{g}(\vdelta,\Delta s)=\mathrm{MLP}\Big(\mathrm{LReLU}\big(\vdelta^{\!\top}
\mW_{\mathrm{bucket}(\Delta s)}+\mathbf{b}_{\mathrm{bucket}(\Delta s)}\big)\Big)\in\mathbb{R}^{K},
\end{equation}
with $\ell$ buckets per layer. Bucketing ties parameters across the geometry space on a
coarse sequence grid while $\vdelta$ modulates within each bucket; a smooth edge gate
$\sigma_{ij}\in(0,1)$ softly downweights edges that are simultaneously far in space and
far in sequence (functional form and per-stage $\ell$ values in
Appendix~\ref{app:arch}).

\paragraph{Outer-product message on hyperedges.}
Fold-Conv propagates in two steps (Figure~\ref{fig:arch}(c)): a vertex-to-edge step
($\mathcal{V}\!\to\!\mathcal{E}$) collects the transformed neighbor content
$\mW_v\vx_j$ onto every edge of both hyperedges, where the rank-$K$ kernel modulates it
through an outer product; an edge-to-vertex step ($\mathcal{E}\!\to\!\mathcal{V}$)
aggregates the readout back to the center residue:
\begin{equation}
\label{eq:hfconv}
\vh_i=\sum_{j\in\mathcal{N}(i)} P_c\Big(\sigma_{ij}\,\mathbf{g}(\vdelta_{ij},\Delta s_{ij})\otimes \mW_v\vx_j\Big),
\end{equation}
where $\mathbf{g}(\cdot)\in\mathbb{R}^{K}$ ($K{=}8$) are edge-conditioned kernel
coefficients, $\sigma_{ij}\in(0,1)$ is the smooth edge gate, $\otimes$ is the outer
product, and $P_c:\mathbb{R}^{K\times C}\!\to\!\mathbb{R}^{C'}$ is a pointwise readout.
Reading $P_c$
channel-wise recovers exactly the matrix-gated form of class~\eqref{eq:matrix}
(Lemma~\ref{lem:outerprod}, Appendix~\ref{app:proofs}):
\begin{equation}
\label{eq:hfkernel}
\vh_i=\sum_{j\in\mathcal{N}(i)} \mW(\vdelta_{ij},\Delta s_{ij})\,\vx_j,
\qquad
\mW(\vdelta,\Delta s)=\sum_{c=1}^{K} \tilde g_c(\vdelta,\Delta s)\,\mB_c ,
\end{equation}
with $\tilde g_c=\sigma\,g_c$ and learned basis operators $\mB_c$ given by the channel
blocks of $P_c$ composed with $\mW_v$.
Eq.~\eqref{eq:hfkernel} is the rank-$K$ factorization anticipated by
Proposition~\ref{prop:rankk}: a full matrix-valued kernel per edge at $O(K\,C\,C')$ rather than
$O(C^{2}d_g)$ cost.

The convolution is wrapped in a bottleneck pre-activation residual unit
following the GearNet block design (details in Appendix~\ref{app:arch}).

\subsection{Two backbone variants}
\label{sec:method:backbone}

\paragraph{\method (residue-level).}
Six Fold-Conv blocks are stacked in two stages of three at a fixed width of $512$, with a
wider sequence window and radius in the deeper stage (Figure~\ref{fig:arch}(d), top;
per-stage configurations in Appendix~\ref{app:arch}). A two-scale readout taps the two
stage outputs (blocks 3 and 6), each an $\mathbb{R}^{L\times 512}$ residue-level feature
map that preserves the full sequence resolution. Pocket detection consumes both levels
separately (\S\ref{sec:method:head}), where every residue must be scored; classification
mean-pools each level and concatenates the pooled vectors into a protein embedding in
$\mathbb{R}^{1024}$.

\paragraph{\methoddeep (protein-level).}
For global prediction tasks, Fold-Conv blocks are interleaved with residue pooling along
a widening channel pyramid ($256\!\to\!512\!\to\!1024\!\to\!2048$ channels;
Appendix~\ref{app:arch}), so the spatial resolution shrinks
($L\!\to\!L/2\!\to\!L/4\!\to\!L/8$) while the channel capacity grows
(Figure~\ref{fig:arch}(d), bottom); a final global pooling produces a protein-level
embedding $\mathbf{f}_{\mathrm{prot}}\in\mathbb{R}^{2048}$ that feeds an MLP classifier
for EC and fold classification.

\subsection{\methodpocket: anchored set prediction for pockets}
\label{sec:method:head}

\begin{figure}[!b]
\centering
\includegraphics[width=1\linewidth]{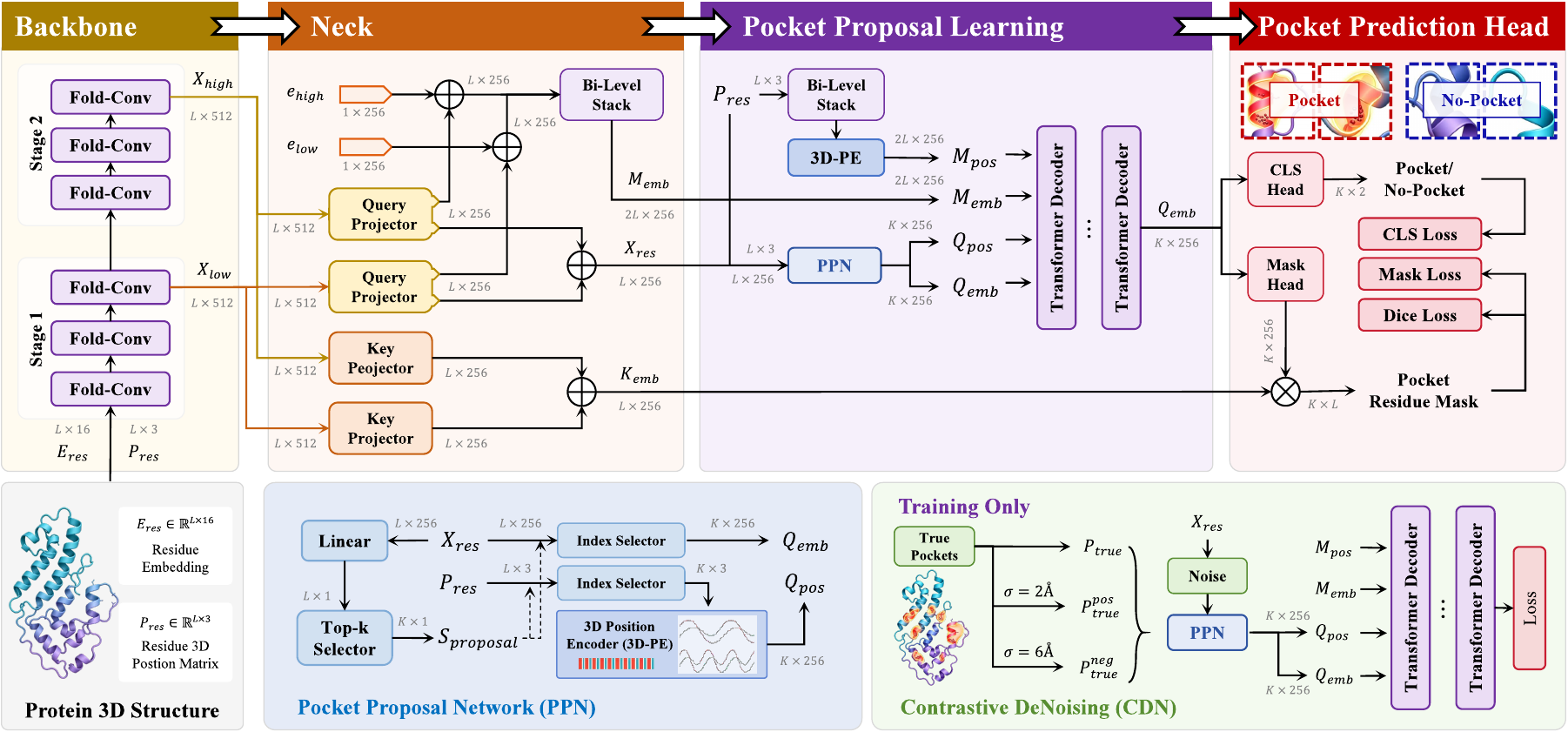}
\caption{\methodpocket pipeline. Two backbone feature levels feed a
query/key/memory neck; the PPN selects structure-anchored queries, refined by a 4-layer
transformer decoder over the bi-level memory; a classification head and a mask head then
read out each pocket. Details in Appendix~\ref{app:arch}.}
\label{fig:head}
\end{figure}

Pocket detection is cast as set prediction \citep{carion2020detr}: the model outputs a
\emph{set} of pockets, each a confidence score paired with a per-residue mask, without
post-processing clustering. \methodpocket builds on the residue-level \method
backbone---every residue must be scored, so the pooling pyramid of \methoddeep does not
apply---and follows a four-stage pipeline (Figure~\ref{fig:head}): a neck projects the
two backbone levels into query, key, and memory features; a pocket proposal network
(PPN) selects structure-anchored queries; a transformer decoder refines them over the
memory; and two heads read out each pocket. Relative to the DETR paradigm, the design
is decoder-only, anchors its queries on the actual 3-D structure, and predicts masks
instead of boxes (detailed comparisons with DETR-style detectors and UniSite are in
Appendix~\ref{app:arch}).

\paragraph{Structure-anchored proposals over a bi-level memory.}
Instead of learning the queries as free parameters, a linear \emph{pocketness} probe
over the fused residue features scores every residue, and the top-$N_q$ ($N_q{=}50$)
become proposals: their own residue features initialize the query content, and their
C$_\alpha$ coordinates, embedded by a three-dimensional Fourier positional encoding
(3D-PE), provide the query position---in the spirit of anchored query selection in
RT-DETR \citep{lv2024rtdetr}, but with anchors that are actual residues. The decoder
memory is likewise structure-preserving: the two backbone levels, tagged by learnable
level embeddings, are stacked into a bi-level memory
$\mM_{\mathrm{emb}}\in\mathbb{R}^{2L\times 256}$, which a 4-layer transformer decoder
cross-attends to directly---with no interposed transformer encoder---injecting the
3D-PE at every layer.

\paragraph{Prediction and training.}
Each refined query $\mathbf{q}_k$ is read out by a classification head
(pocket / no-pocket) and a mask head that takes the inner product with the neck's
residue keys,
$\mathbf{m}_k=\sigma\big(\langle\mathrm{MLP}_{\mathrm{mask}}(\mathbf{q}_k),\,
\mK_{\mathrm{emb}}\rangle\big)\in[0,1]^{L}$, so every proposal independently predicts
its own residue mask. Training matches proposals to ground-truth pockets by Hungarian
matching and sums classification, per-residue mask, and Dice losses, with auxiliary
supervision at every decoder layer (matching weights in Appendix~\ref{app:arch}).
Contrastive denoising (CDN) training \citep{li2022dndetr} appends denoising queries
built from ground-truth pockets---two positive copies (one clean, one lightly noised)
and one heavily noised \emph{contrastive negative} labeled as no-object,
attention-isolated from the real queries---which stabilizes bipartite matching in early
training and teaches the decoder to reject near-miss decoys (noise scales in
Appendix~\ref{app:arch}). The same head and training configuration is shared by all
backbone comparisons of \S\ref{sec:exp}.

\section{Experiments}
\label{sec:exp}

\subsection{Setup}
\label{sec:exp:setup}
\textbf{Tasks and data.} EC number prediction and fold classification follow the GearNet
protocol and splits \citep{zhang2023gearnet}; we report Fmax@50\% and AUPR@95 for EC in
the main text (all five cutoffs in Appendix~\ref{app:ec-full}) and accuracy on the three
fold-classification splits (fold / superfamily / family). Pocket
detection is trained on UniSite-DS \citep{fan2025unisite} and evaluated in-distribution
(AP at IoU 0.3/0.5) and zero-shot on HOLO4K-sc and COACH420 (AP@0.3, DCC, DCA), matching
the UniSite evaluation protocol.

\textbf{Baselines and training.} Classification baselines are GCN, GAT, Point
Transformer, SchNet, GVP, GearNet(-Edge/-IEConv), ProNet(-Backbone), CDConv, and SCHull;
values follow \citet[Tables~2, 6 and~7]{zhang2023gearnet} where available and are our
reproductions under the unified recipe otherwise (provenance and deviations in
Appendix~\ref{app:details}). Pocket baselines are Fpocket, Fpocket-rescore, P2Rank,
DeepPocket, GrASP, VN-EGNN, and UniSite-1D/3D \citep{fan2025unisite}. All models
use AdamW with warmup and cosine decay, EMA, mixed precision, and gradient clipping; full
recipes, hardware, and the latency protocol are in Appendix~\ref{app:details}.

\subsection{Main results}
\label{sec:exp:main}
\paragraph{Function and structure classification.}
\label{sec:exp:cls}
\begin{table}[!h]
\centering
\caption{EC number prediction and fold classification results. EC reports Fmax at the
50\% sequence-identity cutoff and AUPR at 95\% (full table in
Appendix~\ref{app:ec-full}); fold reports test accuracy on three held-out
splits. Rows above/below the rule are generic/protein-specific architectures.
\textbf{Bold}: best; \underline{underlined}: second best (ties included).}
\label{tab:cls}
\small
\renewcommand{\arraystretch}{0.95}
\setlength{\tabcolsep}{4pt}
\begin{tabular}{lccccccc}
\toprule
& & & \multicolumn{2}{c}{EC Prediction} & \multicolumn{3}{c}{Fold Classification} \\
\cmidrule(lr){4-5}\cmidrule(lr){6-8}
Method & Params & Latency & Fmax@50\% & AUPR & Fold & Superfamily & Family \\
\midrule
GCN & 21.9M & 6.2\,ms & 0.246 & 0.319 & 0.168 & 0.213 & 0.828 \\
GAT & 21.9M & 6.9\,ms & 0.247 & 0.339 & 0.124 & 0.165 & 0.727 \\
Point Transformer & 14.5M & 9.7\,ms & 0.276 & 0.346 & 0.181 & 0.211 & 0.689 \\
SchNet & 6.8M & 8.1\,ms & 0.617 & 0.714 & 0.221 & 0.319 & 0.863 \\
\midrule
GVP & 7.1M & 10.0\,ms & 0.347 & 0.467 & 0.274 & 0.459 & 0.938 \\
GearNet & 31.1M & 4.0\,ms & 0.615 & 0.751 & 0.284 & 0.426 & 0.953 \\
GearNet-Edge & 40.7M & 48.6\,ms & 0.694 & 0.835 & 0.440 & 0.667 & 0.991 \\
GearNet-Edge-IEConv & 46.4M & 44.1\,ms & 0.718 & 0.831 & 0.483 & 0.703 & \underline{0.995} \\
ProNet & 1.5M & 9.1\,ms & 0.669 & 0.795 & 0.462 & 0.553 & 0.947 \\
ProNet-Backbone & 4.3M & 11.9\,ms & 0.629 & 0.755 & 0.527 & 0.703 & 0.993 \\
CDConv & 31.0M & 14.2\,ms & 0.777 & 0.859 & \underline{0.568} & 0.765 & \underline{0.995} \\
SCHull & 4.3M & 11.9\,ms & 0.664 & 0.787 & 0.561 & 0.746 & 0.994 \\
\midrule
\method (ours) & 7.7M & 10.5\,ms & \textbf{0.789} & \textbf{0.874} & 0.540 & \underline{0.775} & \textbf{0.996} \\
\methoddeep (ours) & 20.6M & 15.7\,ms & \underline{0.787} & \underline{0.873} & \textbf{0.578} & \textbf{0.794} & \underline{0.995} \\
\bottomrule
\end{tabular}
\end{table}

Table~\ref{tab:cls} compares \method with from-scratch structure encoders. \methoddeep
attains the best accuracy among protein-specific methods on the Fold split---the hardest
generalization setting---surpassing even the IEConv-augmented GearNet variant with a
single convolution type instead of a heterogeneous layer stack, and is also best on the
Superfamily split (Family is near-saturated for all recent methods). On EC, \method
attains the best numbers, with \methoddeep close behind---both ahead of every baseline.
The two variants occupy complementary sweet spots: the flat \method preserves the
residue-level detail that function annotation rewards, while the hierarchical \methoddeep
turns multi-scale pooling into the best global fold representation---at 19--21M
parameters, both remain smaller than the GearNet family and CDConv.

\paragraph{Ligand binding site detection.}
\label{sec:exp:pocket}
\begin{table}[t]
\centering
\caption{Pocket detection: in-distribution (UniSite-DS) and zero-shot generalization
(HOLO4K-sc, COACH420; all methods trained on UniSite-DS only). Left: AP at IoU 0.3/0.5;
right: zero-shot AP@0.3 and DCC/DCA top-$n$ success rates (4\,\AA). \textbf{Bold}:
best; \underline{underlined}: second best.}
\label{tab:pocket}
\small
\renewcommand{\arraystretch}{0.95}
\setlength{\tabcolsep}{2.5pt}
\begin{tabular}{lcc|cccccc}
\toprule
& \multicolumn{2}{c}{UniSite-DS} & \multicolumn{3}{c}{HOLO4K-sc (zero-shot)} & \multicolumn{3}{c}{COACH420 (zero-shot)} \\
\cmidrule(lr){2-3}\cmidrule(lr){4-6}\cmidrule(lr){7-9}
Method & AP@0.3 $\uparrow$ & AP@0.5 $\uparrow$ & AP@0.3 $\uparrow$ & DCC $\uparrow$ & DCA $\uparrow$ & AP@0.3 $\uparrow$ & DCC $\uparrow$ & DCA $\uparrow$ \\
\midrule
Fpocket & 0.184 & 0.102 & 0.271 & 0.308 & 0.438 & 0.211 & 0.271 & 0.411 \\
Fpocket-rescore & 0.508 & 0.235 & 0.590 & 0.518 & 0.765 & 0.560 & 0.441 & 0.711 \\
P2Rank & 0.506 & 0.216 & 0.601 & 0.530 & \underline{0.819} & 0.619 & 0.464 & 0.741 \\
DeepPocket & 0.427 & 0.233 & 0.542 & 0.493 & 0.737 & 0.518 & 0.396 & 0.676 \\
GrASP & 0.447 & 0.285 & 0.667 & 0.513 & 0.742 & 0.715 & 0.485 & \underline{0.762} \\
VN-EGNN & 0.162 & 0.071 & 0.261 & \underline{0.586} & 0.700 & 0.264 & \underline{0.545} & 0.753 \\
UniSite-1D & 0.512 & 0.303 & 0.687 & 0.554 & 0.769 & 0.592 & 0.455 & 0.735 \\
UniSite-3D & \underline{0.560} & \underline{0.384} & \underline{0.709} & 0.572 & 0.788 & \underline{0.720} & 0.470 & 0.738 \\
\midrule
\methodpocket (ours) & \textbf{0.617} & \textbf{0.467} & \textbf{0.735} & \textbf{0.681} & \textbf{0.827} & \textbf{0.768} & \textbf{0.563} & \textbf{0.786} \\
\bottomrule
\end{tabular}
\end{table}

We have three observations. First, \methodpocket sets the best in-distribution result on UniSite-DS
(Table~\ref{tab:pocket}), improving AP@0.3 by 5.7 points over UniSite-3D and AP@0.5 by 8.3
points. Second, the gap widens out of
distribution: \methodpocket exceeds UniSite-3D on all six zero-shot metrics, with the
largest gain on localization quality (DCC: +10.9 on HOLO4K-sc, +9.3 on COACH420). Third,
all of this is achieved with \emph{pure three-dimensional input}: no ESM features, no
evolutionary information.

\subsection{The expressivity ladder is causal: operator-only swaps}
\label{sec:exp:operator}

\begin{table}[!h]
\centering
\caption{Operator-only ablation: all rows share the default \method backbone, task head,
and training recipe; only the aggregation operator varies. Params and latency refer to
the classification configuration. \textbf{Bold}: best;
\underline{underlined}: second best.}
\label{tab:operator}
\renewcommand{\arraystretch}{0.95}
\setlength{\tabcolsep}{3pt}
\begin{tabular}{lcccccc}
\toprule
& & & \multicolumn{2}{c}{EC prediction} & \multicolumn{2}{c}{Pocket detection} \\
\cmidrule(lr){4-5}\cmidrule(lr){6-7}
Operator & Params & Latency & Fmax@50\% & AUPR & AP@0.3 & AP@0.5 \\
\midrule
Additive MP & 3.9M & 9.6\,ms & 0.677 & 0.808 & 0.534 & 0.264 \\
Scalar gate & 3.9M & 11.2\,ms & 0.713 & 0.834 & 0.500 & 0.247 \\
Channel gate & 3.6M & 10.3\,ms & \underline{0.777} & \underline{0.872} & \underline{0.608} & \underline{0.432} \\
Matrix gate (\method) & 6.6M & 10.5\,ms & \textbf{0.789} & \textbf{0.874} & \textbf{0.617} & \textbf{0.467} \\
\bottomrule
\end{tabular}
\end{table}

Table~\ref{tab:operator} makes the expressivity ladder empirical. The two lowest rungs
sit close together, in task-dependent order---a single geometry-conditioned coefficient
per edge, however computed, cannot bind content to geometry. The decisive jump comes
from \emph{per-channel} modulation, which closes most of the gap on EC; yet the full
matrix gate still leads on every metric of both tasks, its margin over channel gating
growing at the stricter IoU threshold ($3.5$ AP@0.5 points)---precisely where
content--geometry binding should matter most.

\begin{figure*}[!h]
\centering
\includegraphics[width=\textwidth]{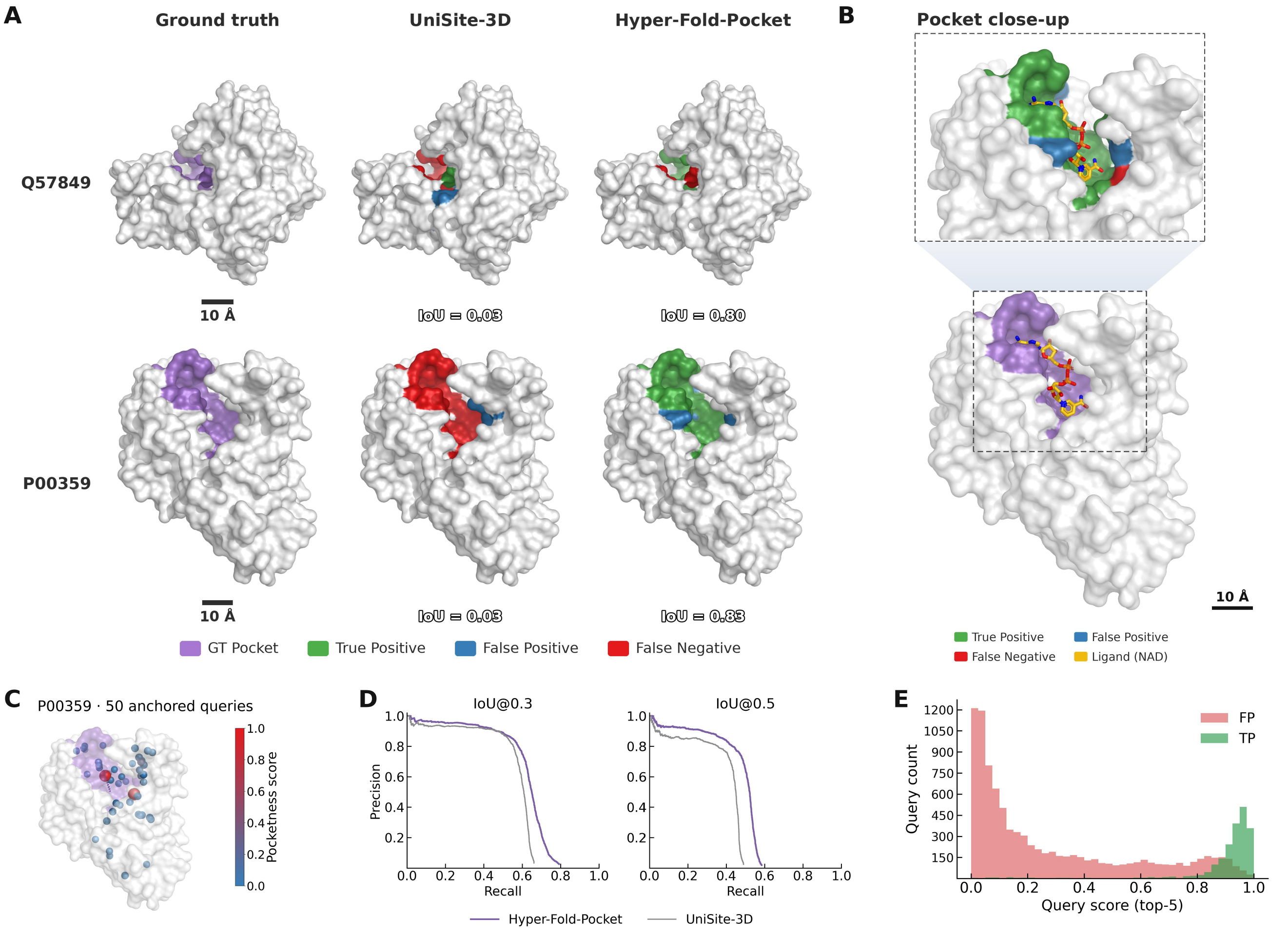}
\caption{Qualitative analysis of ligand-binding-site detection.
\textbf{(A)} Per-residue error maps on Q57849 and P00359 (rows;
columns: ground truth, UniSite-3D, \methodpocket): pocket in purple; TP/FP/FN in
green/blue/red; pocket IoU under each prediction. \textbf{(B)} Pocket anatomy on
P00359: whole protein with ground-truth pocket in purple and the bound NAD ligand; the
close-up uses the same error coloring. \textbf{(C)} Pocketness distribution
of the 50 anchored queries. \textbf{(D)} Precision--recall curves at IoU 0.3/0.5.
\textbf{(E)} Top-5 query scores: high scores are almost exclusively true positives.
Panel protocols in Appendix~\ref{app:pocket-abl}.}
\label{fig:gallery}
\end{figure*}

\subsection{Analysis}
\label{sec:exp:analysis}
\label{sec:exp:efficiency}

Figure~\ref{fig:teaser}(b) plots measured latency against accuracy, marker size
proportional to parameters: \methodpocket (10.07M, 23.9\,ms) is $4.8\times$ faster and
$68\times$ smaller than UniSite-3D (683.9M, 114.3\,ms) at the best accuracy, and the
ESM-2 embedding stage alone (38.1\,ms) costs more than the entire \methodpocket forward
pass---consistent with the theory, as Fold-Conv keeps the
$O(|\mathcal{E}|\,C^2)$ order of additive message passing whereas transformer-based
pipelines pay $O(L^2)$ self-attention over residues
(complexity analysis in Appendix~\ref{app:arch}). The gap is not only asymptotic: the
complete bilinear operator is memory-infeasible at residue resolution, and rank-$K$
factorization is what makes the ceiling usable there. In pocket ablations
(Appendix~\ref{app:pocket-abl}), the
rank-$K$ kernel saturates at $K{=}8$ and contrastive denoising adds $1.5$ AP@0.5.

Figure~\ref{fig:gallery} dissects where the detection gain comes from: UniSite-3D's
errors scatter into false-positive surface patches and missed pocket residues, whereas
\methodpocket's align tightly with the ground truth
(Figure~\ref{fig:gallery}A--B); the anchored queries concentrate their pocketness on
pocket neighborhoods (Figure~\ref{fig:gallery}C); and the top-5 query scores are almost
exclusively true positives (Figure~\ref{fig:gallery}E), mirroring the PR-curve
dominance at both IoU thresholds (Figure~\ref{fig:gallery}D). On multi-pocket proteins,
anchored queries specialize to distinct sites rather than competing for one. The margin
persists under distribution shift: center-localization gains are largest on the zero-shot
benchmarks (DCC $+9$--$11$ over UniSite-3D), suggesting expressive 3D features transfer
better than language-model fusion.

\section{Conclusion}
\label{sec:concl}
The expressive limit of sequence--geometry interaction has a clean answer: the complete
bilinear operator over content--geometry outer products.
Mainstream additive message passing occupies the lowest rung and is provably blind to
content--geometry binding; \method approaches the ceiling at message-passing cost, with
payoff spanning function annotation, fold classification, and zero-shot pocket
detection---without sequence language model features.

\subsection*{AI use statement}

In this work, we used generative AI tools for literature search and retrieval during
background research, for assistance in checking and testing code, and for language
polishing of the manuscript. All AI-assisted outputs were reviewed and verified by the
authors, and all quantitative results were read directly from raw experimental logs. We
take responsibility for the final content of this work, including text, claims, and
artifacts produced with the aid of generative AI.

\subsection*{Ethics statement}

This work studies protein structure modeling on public benchmarks; it involves no human
subjects, sensitive or personally identifiable data, and no foreseeable harmful
applications. The authors have read and adhere to the ICLR Code of Ethics.

\subsection*{Reproducibility statement}

Appendix~\ref{app:arch} specifies all architecture configurations and
Appendix~\ref{app:details} the full training recipes; complete proofs of all
theoretical claims are given in Appendix~\ref{app:proofs}. Anonymized source code,
configuration files, and trained model weights are included in the supplementary
material; the pocket model runs on a single PDB file out of the box.

\bibliography{references}
\bibliographystyle{iclr2027_conference}

\appendix
\section{Extended related work}
\label{app:related}

\paragraph{Geometric deep learning on protein structures.}
Structure-based encoders represent a protein as a residue graph and learn with
SE(3)-invariant or equivariant layers: GVP \citep{jing2021gvp} introduced geometric vector
perceptrons; GearNet \citep{zhang2023gearnet} added edge message passing and remains a
strong multi-task baseline \citep{jamasb2024benchmark}; CDConv \citep{fan2023cdconv} and
GCPNet \citep{morehead2024gcpnet} explore continuous or geometry-complete alternatives.
\S\ref{sec:theory} stratifies these layers by their gating algebra; here we only note the
closest point in expressivity, CDConv, whose matrix-valued kernels match ours in principle
but are memory-intensive at full residue resolution, forcing early downsampling that trades
away exactly the residue-level detail tasks such as pocket detection require. \method keeps
the matrix-gated kernel but factorizes it at rank $K$ over sequence and contact hyperedges,
bringing the cost into the message-passing regime at full sequence length---$4\times$ fewer
parameters and lower latency than CDConv (Table~\ref{tab:cls})---which is precisely what
makes the residue-level \methodpocket head practical.

\paragraph{Sequence models and sequence--structure fusion.}
Protein language models pretrained at evolution scale
\citep{rives2021esm1b,lin2023esm2} dominate sequence-based prediction, and hybrid models
inject structure into them \citep{su2024saprot,hayes2025esm3}. Our results suggest that a
sufficiently expressive 3D backbone can make such fusion unnecessary
(\S\ref{sec:exp:pocket}).

\paragraph{Ligand binding site detection.}
Geometry methods (Fpocket \citep{leguilloux2009fpocket}), ML rescoring
(P2Rank \citep{krivak2018p2rank}), and CNN/GNN detectors
\citep{aggarwal2022deeppocket,smith2024grasp,sestak2024vnegnn} operate on surface or
graph features; UniSite \citep{fan2025unisite} reframed detection as set prediction over
fused sequence--structure features. We adopt the set-prediction formulation but replace
the encoder.

\paragraph{Expressivity of graph networks.}
The WL hierarchy bounds standard MPNNs \citep{xu2019gin}, and invariant tensor networks
characterize maximal universal layers \citep{maron2019invariant}; on hypergraphs, the
expressive power of hypergraph neural networks has been characterized via generalized
Weisfeiler--Lehman tests \citep{feng2026howpowerful}. Our analysis is
orthogonal and finer-grained: within the \emph{geometry-conditioned} layers actually used
on proteins, we stratify expressivity by the algebraic form of the content--geometry
interaction, in the spirit of factorization-machine analyses of feature interaction
\citep{guo2017deepfm}.

\paragraph{Hypergraph modeling.}
Hypergraph neural networks \citep{feng2019hgnn,yadati2019hypergcn} generalize pairwise
edges to hyperedges that connect a vertex to a \emph{set} of related vertices, with
HGNN$^+$ \citep{gao2023hgnnplus} unifying vertex- and hyperedge-level propagation, and
hypergraph isomorphism computation \citep{feng2024hyperiso} characterizing structural
equivalence of such high-order structures. Hypergraph modeling has also proven effective
in detection: Hyper-YOLO \citep{feng2025hyperyolo} constructs hypergraphs over visual
features to capture high-order correlations among candidate regions. Hyperedges are a
natural language for the two residue groupings we use---sequence adjacency and spatial
contact. \method borrows this grouping view only: each residue anchors a sequence
hyperedge and a contact hyperedge, but information propagation is carried out by our own
matrix-gated convolution (\S\ref{sec:method:layer}), not by hypergraph propagation rules.

\section{Architecture details}
\label{app:arch}

\subsection{\method backbone}

The backbone maps the residue-level input (amino-acid embedding, C$_\alpha$ coordinates,
and local frames; \S\ref{sec:method:hyper}) to residue- or protein-level features by
stacking Fold-Conv blocks on the radius graph with its two hyperedge types. Every block
runs the same pipeline: the geometric feature builder computes $\vdelta_{ij}$ and
$\Delta s_{ij}$ for each edge (Eq.~\eqref{eq:edgefeat}), the geometric kernel network
(GK-Net) turns them into rank-$K$
kernel coefficients (Eq.~\eqref{eq:weightnet}), the outer-product message of
Eq.~\eqref{eq:hfconv} aggregates over both hyperedges, and a bottleneck residual unit
closes the block. The two variants differ only in how the blocks are staged
(\S\ref{sec:method:backbone}). The remaining implementation details follow.

\paragraph{Fold-Conv residual block.}
The convolution of Eq.~\eqref{eq:hfconv} is wrapped in a bottleneck pre-activation
residual unit,
$\vh^{(b+1)}=\vh^{(b)}+W_{\mathrm{out}}\,\mathrm{Drop}\big(\mathrm{Conv}_{r_b}
\big(\mathrm{LReLU}(\mathrm{BN}(W_{\mathrm{in}}\vh^{(b)}))\big)\big)$,
with bottleneck width $C/2$ and BN momentum $0.2$, following the GearNet block design.

\paragraph{GK-Net bucketing and smooth gate.}
The bucketed first layer of GK-Net (Eq.~\eqref{eq:weightnet}) uses $\ell$ offset buckets
per stage ($\ell{=}17$ in the shallow stage of \method, $\ell{=}21$ in the deeper stage
and in \methoddeep, which uses $\ell{=}5$ for fold); contact neighbors, whose sequence displacement saturates
the clamp, share the boundary bucket. The smooth edge gate is
$\sigma_{ij}=\tfrac12\big[1-\tanh\big(16\,\bar d\,\bar s-14\big)\big]$ with
$\bar d=\|\mathbf{p}_i-\mathbf{p}_j\|/r$ and $\bar s=|\Delta s_{ij}|/(\ell/2)$, softly
downweighting edges that are simultaneously far in space and far in sequence.

\paragraph{Backbone configurations.}
\method stacks six Fold-Conv blocks in two stages of three blocks each, with
$(\ell, r) = (17, 8\,\text{\AA})$ and $(21, 10\,\text{\AA})$ respectively, at a constant
width of 512 (bottleneck blocks with base width 32). The readout taps the two stage
outputs (blocks 3 and 6): classification mean-pools each level and concatenates the
pooled vectors into $\mathbb{R}^{1024}$, while the pocket neck consumes the two levels
separately (the \methodpocket paragraph below). \methoddeep is a hierarchical variant of \method with a
four-stage widening pyramid of channels $256 \rightarrow 512 \rightarrow 1024 \rightarrow
2048$: each stage contains two Fold-Conv blocks with radius $r = 8, 12, 16,
20\,\text{\AA}$ per stage, followed by residue pooling that progressively shortens the
sequence (19--21M parameters depending on the task head). The sequence window follows
the per-task convention of the reference protocol: $\ell = 21$ for EC and $\ell = 5$
for fold, where we find a wide window overfits (test fold accuracy $0.578 \to 0.493$
at $\ell{=}21$).

\paragraph{Complexity.}
Per Fold-Conv layer, evaluating the rank-$K$ kernel of Eq.~\eqref{eq:hfkernel} costs
$O(|\mathcal{E}|\,K\,C\,C')$ with bottleneck width $C'{=}C/2$ and small constant
$K{=}8$---the same order as additive message passing's $O(|\mathcal{E}|\,C^2)$, up to a
small constant factor. The graph itself is sparse: each residue aggregates from a fixed
window of $\ell$ sequence neighbors plus its radius neighbors, whose count is bounded by
the near-constant packing density of C$_\alpha$ atoms, so
$|\mathcal{E}|=O\big(L(\ell{+}\rho)\big)$ with $\ell,\rho$ independent of $L$ and the
per-layer cost grows \emph{linearly} with sequence length. Transformer-based pipelines
such as UniSite instead pay $O(L^2 d)$ self-attention over all residue pairs on top of
the backbone. The pocket head preserves the linear scaling: with the proposal count
$N_q$ fixed, decoder cross-attention over the $2L$ memory rows and the mask inner
products both cost $O(L\,N_q\,d)$. Measured latency confirms the gap to additive
message passing is a small constant factor, while remaining orders of magnitude cheaper
than protein-language-model-based pipelines (\S\ref{sec:exp:efficiency}).

\subsection{\methodpocket}

\methodpocket turns the residue-level features of \method into a set of pocket
predictions in four stages (\S\ref{sec:method:head}): a neck projects the two backbone
levels into query, key, and memory features; the pocket proposal network (PPN) scores
every residue and selects
the top-$N_q$ structure-anchored proposals; a 4-layer transformer decoder refines them
against the bi-level memory; and two heads produce, per proposal, a pocket score and a
residue mask. Training combines Hungarian-matched classification/mask/Dice losses with
contrastive denoising. The complete design follows.

\paragraph{Neck.}
The residue-level backbone runs at width $512$ and is read out at blocks $\{3,6\}$,
giving $\mX_{\mathrm{low}},\mX_{\mathrm{high}}\in\mathbb{R}^{L\times 512}$; its node
input is a learned $16$-dimensional amino-acid-type embedding. Because the structure
may not cover the full UniProt sequence, backbone features are scattered onto the
UniProt residue axis of length $L$: positions without structural coverage receive a
shared learned null vector, and a binary mask records coverage. Two pairs of linear
projectors then produce the fused residue features and the mask keys,
\begin{equation}
\mX_{\mathrm{res}}=\mX_{\mathrm{low}}\mW_{\mathrm{low}}^{q}+\mX_{\mathrm{high}}\mW_{\mathrm{high}}^{q},
\qquad
\mK_{\mathrm{emb}}=\mX_{\mathrm{low}}\mW_{\mathrm{low}}^{k}+\mX_{\mathrm{high}}\mW_{\mathrm{high}}^{k},
\end{equation}
with $\mW_{\cdot}^{q},\mW_{\cdot}^{k}\in\mathbb{R}^{512\times 256}$, both outputs in
$\mathbb{R}^{L\times 256}$. For the decoder memory, each level's query projection is
tagged with a learnable level embedding
$\mathbf{e}_{\mathrm{low}},\mathbf{e}_{\mathrm{high}}\in\mathbb{R}^{256}$ and the two
levels are concatenated along the residue axis into the bi-level memory
$\mM_{\mathrm{emb}}=[\,\mX_{\mathrm{low}}\mW_{\mathrm{low}}^{q}+\mathbf{e}_{\mathrm{low}};\
\mX_{\mathrm{high}}\mW_{\mathrm{high}}^{q}+\mathbf{e}_{\mathrm{high}}\,]\in\mathbb{R}^{2L\times256}$,
paired with a shared three-dimensional Fourier positional encoding (3D-PE)
$\mM_{\mathrm{pos}}$ (the same residue
coordinates, repeated for both levels) and a padding mask that excludes invalid and
unmapped positions.

\paragraph{PPN.}
A linear probe $s_i=\mathbf{w}^{\top}\mathbf{x}_{\mathrm{res},i}$ with
$\mathbf{w}\in\mathbb{R}^{256}$ scores every residue (\emph{pocketness}), supervised by
a residue-level binary cross-entropy against the union of ground-truth masks (loss
weight $2$). The top-$N_q$ eligible residues ($N_q{=}50$; eligible $=$ structurally
mapped and valid) become proposals: their rows of $\mX_{\mathrm{res}}$ initialize the
query content $\mQ_{\mathrm{emb}}\in\mathbb{R}^{N_q\times256}$, and their C$_\alpha$
coordinates serve as anchors $\{\mathbf{a}_k\}_{k=1}^{N_q}$, embedded by the 3D-PE
below into $\mQ_{\mathrm{pos}}$. Proteins with fewer than $N_q$ eligible residues pad
the remaining queries with a learned null content vector.

\paragraph{3D Fourier positional encoding.}
All coordinates are protein-centered (each structure's C$_\alpha$ centroid is
subtracted), so the encoding is translation-invariant like the backbone. A coordinate
$\mathbf{a}\in\mathbb{R}^{3}$ is embedded per axis with $8$ frequency bands,
$\gamma(\mathbf{a})=[\sin(2^{m}\mathbf{a}),\cos(2^{m}\mathbf{a})]_{m=0}^{7}\in\mathbb{R}^{48}$,
followed by a learned linear projection to $256$ dimensions. This yields the query
positional encoding $\mQ_{\mathrm{pos}}$ (of the anchors) and the memory positional
encoding $\mM_{\mathrm{pos}}$ (of every residue).

\paragraph{Decoder.}
A 4-layer transformer decoder ($8$ attention heads, feed-forward width $1024$, dropout
$0.1$, post-norm with a final LayerNorm) refines $\mQ_{\mathrm{emb}}$ by
cross-attending to the bi-level memory $(\mM_{\mathrm{emb}},\mM_{\mathrm{pos}})$, with
$\mQ_{\mathrm{pos}}$ injected as the query positional encoding at every layer. The
output of each decoder layer is read out by the prediction heads below, providing
auxiliary supervision.

\paragraph{Prediction heads.}
Each refined query $\mathbf{q}_k$ is read out by a two-layer classification MLP
$z_k=\mathrm{MLP}_{\mathrm{cls}}(\mathbf{q}_k)\in\mathbb{R}^{2}$ (pocket / no-pocket)
and a three-layer mask MLP followed by an inner product with the mask keys,
$\mathbf{m}_k=\sigma\big(\langle\mathrm{MLP}_{\mathrm{mask}}(\mathbf{q}_k),\,
\mK_{\mathrm{emb}}\rangle\big)\in[0,1]^{L}$, so every proposal independently predicts
its own residue mask. At inference, masks are binarized at $0.5$, and the pocket score
is the classification probability multiplied by the mean mask probability within the
predicted mask---a calibration factor that downweights diffuse masks.

\paragraph{Matching and losses.}
Hungarian matching between the $N_q{=}50$ proposals and the ground-truth pockets
minimizes
$\mathcal{C}=2\,\mathcal{C}_{\mathrm{cls}}+5\,\mathcal{C}_{\mathrm{bce}}+5\,\mathcal{C}_{\mathrm{dice}}$,
where $\mathcal{C}_{\mathrm{cls}}$ is the negative predicted probability of the
ground-truth class, $\mathcal{C}_{\mathrm{bce}}$ the per-residue binary cross-entropy
between the predicted and ground-truth masks, and $\mathcal{C}_{\mathrm{dice}}$ the
Dice discrepancy. The training loss sums the same three terms over matched pairs
(unmatched queries contribute only the classification term toward the no-pocket class),
normalized by the number of ground-truth pockets, with auxiliary supervision at every
decoder layer; the PPN residue-level loss (weight $2$) is added on top.

\paragraph{Contrastive denoising (CDN).}
Each ground-truth pocket spawns denoising queries whose content is its mask-pooled
residue feature, $\mathbf{c}=\sum_{i}w_i\,\mathbf{x}_{\mathrm{res},i}$ with $w_i$ the
ground-truth mask normalized over mapped residues, and whose anchor is the masked
C$_\alpha$ centroid $\sum_{i}w_i\,\mathbf{a}_i$. Two positive copies---one clean, one
with Gaussian feature noise $0.5$ and coordinate noise $\sigma{=}2$\,\AA---reconstruct
their source pocket, while one \emph{contrastive negative} copy with larger noise
(feature noise $\times2$, $\sigma{=}6$\,\AA) is labeled no-object; the number of DN
queries is capped at $32$. DN queries attend to the same memory but are
attention-isolated from the real queries, skip Hungarian matching (each is aligned 1:1
with its source pocket), and incur mask losses on the positive copies only.

\paragraph{Relation to DETR-style detectors.}
\methodpocket borrows the set-prediction formulation of DETR \citep{carion2020detr} but
redesigns every component for three-dimensional structures. (i) \emph{No transformer
encoder}: DETR-family detectors interpose a transformer encoder between the
convolutional backbone and the decoder; we replace it with a lightweight projection
neck over the two backbone levels, so the decoder cross-attends directly to a
convolutional, geometry-aware memory. (ii) \emph{Structure-anchored queries}: DETR's
queries are freely learned embeddings, and anchor-based variants
\citep{liu2022dabdetr,zhang2022dino,lv2024rtdetr} parameterize or select 2-D anchor
boxes from encoder outputs; our proposals are anchored on actual residues---content
initialized from the top-pocketness residue features, position from their C$_\alpha$
coordinates---fixing both the query content and its 3-D reference point without any
encoder. (iii) \emph{Masks instead of boxes}: each proposal predicts a per-residue mask
by inner product with the residue keys, rather than regressing a bounding box.
(iv) \emph{Contrastive denoising}: DN-DETR and DINO \citep{li2022dndetr,zhang2022dino}
reconstruct noised ground truth as an auxiliary positive-only task; we additionally
label heavily noised copies as \emph{negatives}, teaching the decoder to reject
near-miss decoys around real pockets. The bi-level memory plays the role of the
multi-scale encoder features of Deformable-DETR \citep{zhu2021deformable}, but comes
from two depths of a single convolutional backbone rather than an image pyramid.

\paragraph{Relation to UniSite \citep{fan2025unisite}.}
UniSite-1D/3D instantiate the DETR template almost verbatim for pocket detection: a
6-layer transformer encoder processes fused ESM--GearNet-Edge residue features, and a
6-layer decoder refines $32$ \emph{freely learned} site queries (zero-initialized
content with learned positional embeddings), read out by classification and
inner-product mask heads and trained with Hungarian-matched BCE and Dice losses only.
\methodpocket shares the set-prediction formulation but differs in three ways.
(i) \emph{Proposals are discovered, not learned}: our queries are the top-$N_q$ residues
by pocketness, carrying the selecting residues' own features as content and their
C$_\alpha$ coordinates as 3D-PE anchors, whereas UniSite's queries are data-independent
parameters that must localize pockets through cross-attention alone. (ii) \emph{No
transformer encoder}: UniSite interposes six encoder layers between the convolutional
backbone and the decoder; our decoder acts directly on the bi-level convolutional
memory, in which sequence--geometry interactions are already encoded by the backbone.
(iii) \emph{Contrastive denoising}: UniSite trains with matched losses only, while CDN
adds positive reconstructions and near-miss negatives ($+1.5$ AP@0.5,
Figure~\ref{fig:app-readout-head}). A fourth, orthogonal difference is that \methodpocket is
purely structural: its residue features come from \method without any sequence model,
whereas UniSite-3D leans on ESM features fused with GearNet-Edge.

\section{Experimental setup}
\label{app:details}

\paragraph{Classification baselines: provenance and deviations.}
GAT, Point
Transformer, SchNet, GVP and ProNet are our reproductions under the same recipe (ProNet at
backbone level, the variant its authors found best for function prediction); the GCN
efficiency numbers are measured on a reference build mirroring the GAT configuration, as
\citeauthor{zhang2023gearnet} release no official GCN recipe. Point Transformer uses a
stabilized recipe (lr $5{\times}10^{-5}$, gradient clip $1.0$, no AMP), as the unified
protocol diverges for this model. ProNet-Backbone \citep{wang2023pronet} and SCHull
\citep{wang2025schull} modify the input \emph{graph construction}, an axis orthogonal to
the layer operator we study; their efficiency numbers are measured on the official
SCHull4Science configuration (backbone level, three blocks of width 256)---SCHull adds no
parameters, and its hull construction is a preprocessing step. CDConv is the official
model trained with its published fold-specific recipe (SGD, 400 epochs, geometry
parameter $l{=}11$); the EC numbers of CDConv and \method are from our unified-protocol runs, while
their fold numbers use the fold-specific recipes.

\paragraph{EC number prediction.}
All EC models are trained with AdamW (learning rate $10^{-4}$, no weight decay), using a
10-epoch linear warmup followed by cosine decay, a batch size of 16, and 300 epochs in
total; we apply EMA with decay $0.999$, mixed precision, and gradient clipping at norm
$10$.

\paragraph{Fold classification.}
The same AdamW recipe and 300-epoch schedule as EC. For \method and \methoddeep we
additionally use an SGD recipe with EMA, label smoothing, and structure-forcing
augmentation (coordinate noise $\sigma{=}0.15$\,\AA, residue-type masking $p{=}0.15$);
\methoddeep uses the widening channel pyramid described in Appendix~\ref{app:arch},
with the narrower sequence window $\ell{=}5$.

\paragraph{Pocket detection.}
\methodpocket and all re-trained pocket baselines are optimized with AdamW (learning
rate $2{\times}10^{-4}$, weight decay $0.05$ applied only to parameters with at least
two dimensions), using a 3-epoch warmup followed by cosine decay, a batch size of 16,
and 40 epochs in total; we apply EMA with decay $0.999$, clip gradients at norm $0.5$,
and add Gaussian coordinate noise ($\sigma{=}0.2$\,\AA) to the C$\alpha$ positions
during training.

\paragraph{Latency measurement.}
All models are trained and benchmarked on a single A100-80GB; latency is measured at batch
size 1 on a synthetic length-300 protein (10 warmup and 50 timed forward passes, fp32),
excluding PDB parsing and ESM tokenization. Graph construction follows each model's native
pipeline: it is part of the timed forward pass for models that build graphs on-device
(\method, UniSite, ProNet, CDConv), and is performed once outside the timed region for
models whose framework treats the graph as a preprocessing artifact (torchdrug-based GCN,
GAT and GVP). For the pocket baselines of Table~\ref{tab:pocket}: UniSite-1D / UniSite-3D
are 662.9M / 683.9M parameters at 46.4 / 114.3\,ms (including the ESM-2 embedding stage,
which alone costs 38.1\,ms); DeepPocket is timed as its two-stage pipeline (a 25.9M
segmentation U-Net at 5.8\,ms plus a 0.7M candidate-ranking classifier at 0.6\,ms per
candidate center); VN-EGNN consumes precomputed ESM features that are not counted in its
parameter or latency figures; Fpocket, Fpocket-rescore and P2Rank are geometry-based CPU
tools with no neural parameters and are not timed on GPU.

\section{More evaluation on the \method backbone}
\subsection{Full EC results}
\label{app:ec-full}

Table~\ref{tab:ec-full} reports the full EC benchmark behind the two summary columns of
Table~\ref{tab:cls}: test Fmax at five sequence-identity cutoffs, together with
micro-AUPR at the 95\% cutoff. \method is the best method at every cutoff, and its
advantage over the strongest baseline widens as the cutoff tightens, from $+1.1$ Fmax
points at 95\% ($0.864$ vs.\ $0.853$) to $+2.1$ points at 30\% ($0.726$ vs.\ $0.705$)---the
gain concentrates on the hard, remote-homology cases where geometry must substitute for
sequence similarity. \methoddeep trails \method slightly but stays ahead of every
baseline at all five cutoffs.

\begin{table}[h]
\centering
\caption{EC number prediction: test Fmax at five sequence-identity cutoffs
(30/40/50/70/95\%) and micro-AUPR at the 95\% cutoff. Method order, provenance, and the
Params/Latency measurements are as in Table~\ref{tab:cls}. \textbf{Bold}: best overall;
\underline{underlined}: second best overall.}
\label{tab:ec-full}
\small
\renewcommand{\arraystretch}{1.22}
\begin{tabular*}{\textwidth}{@{\extracolsep{\fill}}lcccccccc}
\toprule
 & & & \multicolumn{5}{c}{Fmax at sequence-identity cutoff} & AUPR \\
\cmidrule(lr){4-8}
Method & Params & Latency & 30\% & 40\% & 50\% & 70\% & 95\% & 95\% \\
\midrule
\rowcolor{gray!12}\multicolumn{9}{l}{\textit{Generic architectures}} \\
GCN & 21.9M & 6.2\,ms & 0.245 & 0.246 & 0.246 & 0.280 & 0.320 & 0.319 \\
GAT & 21.9M & 6.9\,ms & 0.246 & 0.248 & 0.246 & 0.276 & 0.329 & 0.339 \\
Point Transformer & 14.5M & 9.7\,ms & 0.247 & 0.260 & 0.276 & 0.320 & 0.377 & 0.346 \\
SchNet & 6.8M & 8.1\,ms & 0.555 & 0.573 & 0.617 & 0.686 & 0.742 & 0.714 \\
\midrule
\rowcolor{gray!12}\multicolumn{9}{l}{\textit{Protein-specific architectures}} \\
GVP & 7.1M & 10.0\,ms & 0.347 & 0.338 & 0.347 & 0.400 & 0.480 & 0.467 \\
GearNet & 31.1M & 4.0\,ms & 0.557 & 0.570 & 0.615 & 0.693 & 0.730 & 0.751 \\
GearNet-Edge & 40.7M & 48.6\,ms & 0.625 & 0.646 & 0.694 & 0.757 & 0.810 & 0.835 \\
GearNet-Edge-IEConv & 46.4M & 44.1\,ms & 0.644 & 0.674 & 0.718 & 0.774 & 0.813 & 0.831 \\
ProNet & 1.5M & 9.1\,ms & 0.582 & 0.621 & 0.669 & 0.726 & 0.765 & 0.795 \\
ProNet-Backbone & 4.3M & 11.9\,ms & 0.537 & 0.577 & 0.629 & 0.694 & 0.737 & 0.755 \\
CDConv & 31.0M & 14.2\,ms & 0.705 & 0.737 & 0.777 & 0.823 & 0.853 & 0.859 \\
SCHull & 4.3M & 11.9\,ms & 0.562 & 0.608 & 0.664 & 0.729 & 0.771 & 0.787 \\
\midrule
\rowcolor{gray!12}\multicolumn{9}{l}{\textit{Ours}} \\
\method & 7.7M & 10.5\,ms & \textbf{0.726} & \textbf{0.752} & \textbf{0.789} & \textbf{0.832} & \textbf{0.864} & \textbf{0.874} \\
\methoddeep & 20.6M & 15.7\,ms & \underline{0.716} & \underline{0.746} & \underline{0.787} & \underline{0.830} & \underline{0.861} & \underline{0.873} \\
\bottomrule
\end{tabular*}
\end{table}

\subsection{Architecture and geometry ablations}
\label{app:ec-arch}

Figure~\ref{fig:app-ec-heatmap} ablates the backbone's geometry schedule, sequence
window, and width on the same benchmark (configurations in
Table~\ref{tab:app-ec-arch}). All runs use the identical 300-epoch recipe
(AdamW $10^{-4}$, cosine schedule, EMA 0.999) as the EC main table; numbers are test
Fmax in percent at five sequence-identity cutoffs.

\begin{table}[h]
\centering
\caption{Configurations of the EC ablation variants visualized in
Figure~\ref{fig:app-ec-heatmap}: the per-layer radius schedule $r$ (in \AA) and the
sequence window $\ell$. All rows share the six-layer \method recipe (6.67M parameters,
bottleneck $C/2$) and the unified 300-epoch protocol; only the stated schedules vary.
Cell shading encodes magnitude within each schedule (darker = larger).}
\label{tab:app-ec-arch}
\small
\renewcommand{\arraystretch}{1.22}
\setlength{\tabcolsep}{0pt}
\definecolor{rcol}{RGB}{70,130,180}
\definecolor{lcol}{RGB}{0,150,136}
\newcolumntype{C}{>{\centering\arraybackslash}p{0.055\textwidth}}
\begin{tabular}{ll*{12}{C}}
\toprule
 & & \multicolumn{6}{c}{$r$ (\AA)} & \multicolumn{6}{c}{$\ell$} \\
\cmidrule(lr){3-8}\cmidrule(lr){9-14}
Group & Variant & L1 & L2 & L3 & L4 & L5 & L6 & L1 & L2 & L3 & L4 & L5 & L6 \\
\midrule
\multirow{3}{*}{\textit{Radii}}
 & Uniform Radii & \cellcolor{rcol!30}10 & \cellcolor{rcol!30}10 & \cellcolor{rcol!30}10 & \cellcolor{rcol!30}10 & \cellcolor{rcol!30}10 & \cellcolor{rcol!30}10 & \cellcolor{lcol!30}21 & \cellcolor{lcol!30}21 & \cellcolor{lcol!30}21 & \cellcolor{lcol!30}21 & \cellcolor{lcol!30}21 & \cellcolor{lcol!30}21 \\
 & Expanding Radii ($8{\to}12$\,\AA) & \cellcolor{rcol!15}8 & \cellcolor{rcol!15}8 & \cellcolor{rcol!30}10 & \cellcolor{rcol!30}10 & \cellcolor{rcol!45}12 & \cellcolor{rcol!45}12 & \cellcolor{lcol!15}17 & \cellcolor{lcol!15}17 & \cellcolor{lcol!15}17 & \cellcolor{lcol!30}21 & \cellcolor{lcol!30}21 & \cellcolor{lcol!30}21 \\
 & Contracted Radii ($6{\to}12$\,\AA) & \cellcolor{rcol!0}6 & \cellcolor{rcol!0}6 & \cellcolor{rcol!15}8 & \cellcolor{rcol!15}8 & \cellcolor{rcol!30}10 & \cellcolor{rcol!45}12 & \cellcolor{lcol!0}13 & \cellcolor{lcol!0}13 & \cellcolor{lcol!15}17 & \cellcolor{lcol!15}17 & \cellcolor{lcol!30}21 & \cellcolor{lcol!30}21 \\
\midrule
\multirow{2}{*}{\textit{Window}}
 & Uniform Window & \cellcolor{rcol!15}8 & \cellcolor{rcol!15}8 & \cellcolor{rcol!15}8 & \cellcolor{rcol!30}10 & \cellcolor{rcol!30}10 & \cellcolor{rcol!30}10 & \cellcolor{lcol!30}21 & \cellcolor{lcol!30}21 & \cellcolor{lcol!30}21 & \cellcolor{lcol!30}21 & \cellcolor{lcol!30}21 & \cellcolor{lcol!30}21 \\
 & Contracted + Long Window & \cellcolor{rcol!0}6 & \cellcolor{rcol!0}6 & \cellcolor{rcol!15}8 & \cellcolor{rcol!30}10 & \cellcolor{rcol!45}12 & \cellcolor{rcol!45}12 & \cellcolor{lcol!0}13 & \cellcolor{lcol!0}13 & \cellcolor{lcol!15}17 & \cellcolor{lcol!15}17 & \cellcolor{lcol!30}21 & \cellcolor{lcol!45}25 \\
\midrule
\textit{Width} & Quarter Width ($C/4$) & \cellcolor{rcol!15}8 & \cellcolor{rcol!15}8 & \cellcolor{rcol!15}8 & \cellcolor{rcol!30}10 & \cellcolor{rcol!30}10 & \cellcolor{rcol!30}10 & \cellcolor{lcol!15}17 & \cellcolor{lcol!15}17 & \cellcolor{lcol!15}17 & \cellcolor{lcol!30}21 & \cellcolor{lcol!30}21 & \cellcolor{lcol!30}21 \\
\midrule
\textit{Default} & \method & \cellcolor{rcol!15}8 & \cellcolor{rcol!15}8 & \cellcolor{rcol!15}8 & \cellcolor{rcol!30}10 & \cellcolor{rcol!30}10 & \cellcolor{rcol!30}10 & \cellcolor{lcol!15}17 & \cellcolor{lcol!15}17 & \cellcolor{lcol!15}17 & \cellcolor{lcol!30}21 & \cellcolor{lcol!30}21 & \cellcolor{lcol!30}21 \\
\bottomrule
\end{tabular}
\end{table}

\begin{figure}[t]
\centering
\includegraphics[width=0.9\textwidth]{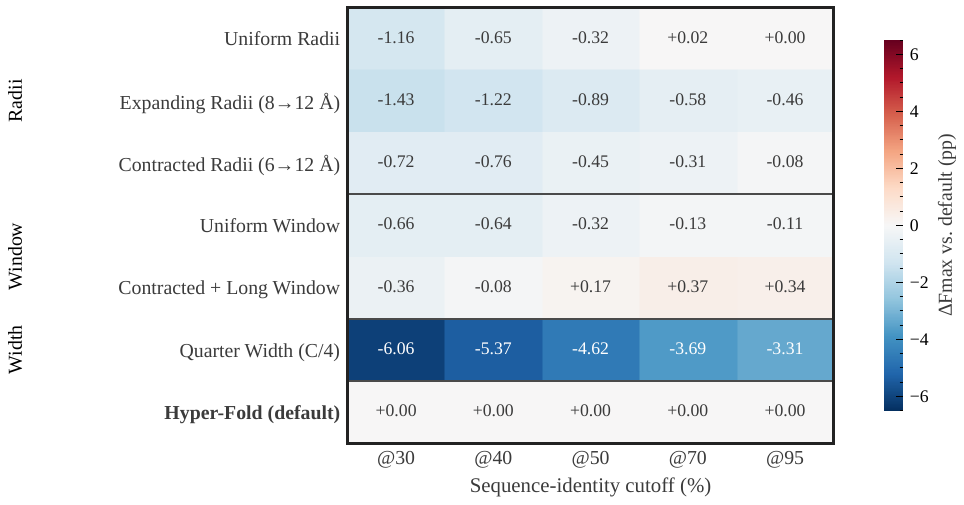}
\caption{EC architecture and geometry ablations as a $\Delta$Fmax heatmap:
each cell is the variant's test Fmax minus the default \method configuration's at the
same sequence-identity cutoff (configurations in Table~\ref{tab:app-ec-arch}). Rows are
grouped by ablated axis; only the quarter-width variant departs markedly from the
default.}
\label{fig:app-ec-heatmap}
\end{figure}

Three findings emerge from Figure~\ref{fig:app-ec-heatmap}. First, the shallow-layer
geometry matters more than the deep-layer one: shrinking the shallow radii from 10\,\AA\
to 8\,\AA\ is free (Uniform Radii vs.\ default: $78.55$ vs.\ $78.87$ Fmax@50), and
reducing them further to 6\,\AA\ is harmful (Contracted Radii)---8\,\AA\ is the
shallow-layer sweet spot, consistent with the local chemistry of residue neighborhoods.
Second, width is the binding constraint: a quarter-width model collapses by $4.6$ Fmax@50
points ($74.25$ vs.\ $78.87$), while the bottleneck $C/2$ design matches full width, so
expressivity should be bought with the rank-$K$ kernel rather than raw channels. Third,
the window and radius schedules interact: combining contracted shallow radii with a
longer deep-layer window (Contracted + Long Window) trades strict-cutoff accuracy for
lenient-cutoff gains, whereas the default schedule is the strongest at the strict
30/40\% cutoffs---the remote-homology regime where geometry must substitute for
sequence similarity.

\section{More evaluation on \methodpocket}
\label{app:pocket-abl}

\subsection{Pocket ablations}

\begin{figure}[t]
\centering
\includegraphics[width=\textwidth]{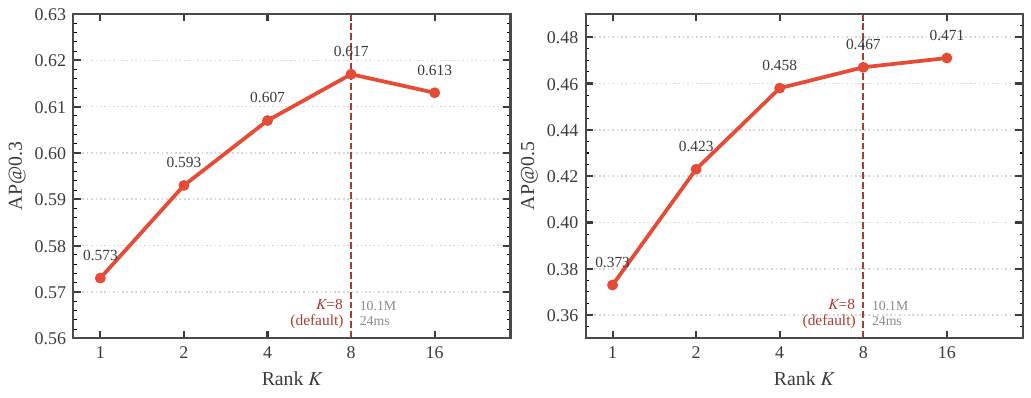}
\caption{Rank-$K$ saturation of \methodpocket: AP@0.3 (left) and AP@0.5 (right) of the
rank sweep, plotted separately because the two metrics
live on different scales. Accuracy saturates at $K{=}8$; parameters and latency are
annotated at the working point. The operator spectrum that explains the saturation is
in Figure~\ref{fig:app-spectrum}.}
\label{fig:app-rankk-sweep}
\end{figure}

\begin{figure}[t]
\centering
\includegraphics[width=\textwidth]{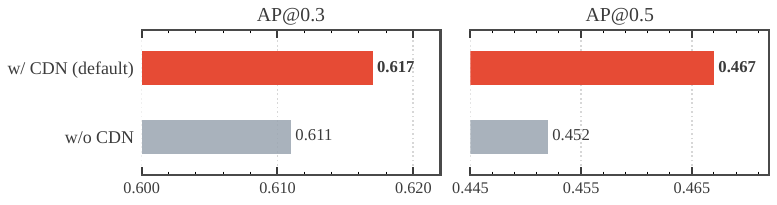}
\caption{Head-component ablation, split by metric (left:
AP@0.3, right: AP@0.5). Coral bars mark the default configuration: removing CDN costs
$0.6$/$1.5$ points (AP@0.3/AP@0.5) at identical latency.}
\label{fig:app-readout-head}
\end{figure}

\begin{figure}[t]
\centering
\includegraphics[width=\textwidth]{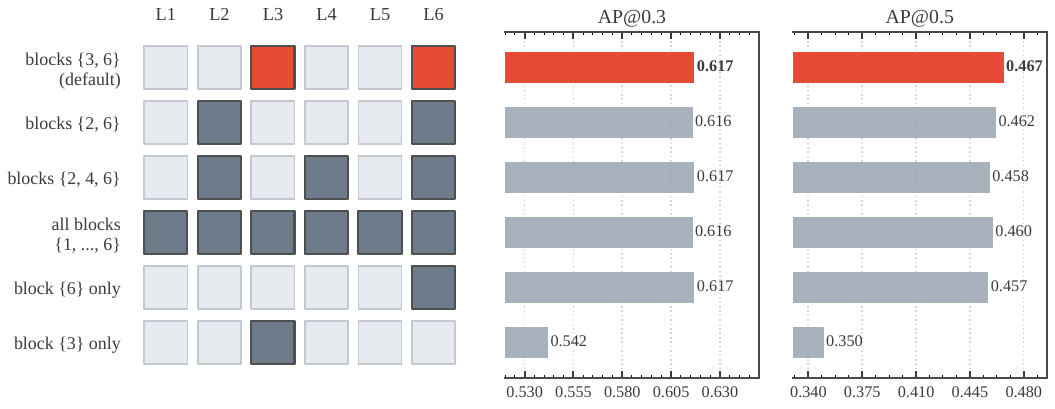}
\caption{Readout selection of \methodpocket. Left: which block outputs feed the
detection neck (coral: selected). Right: AP@0.3 and AP@0.5 of the readout ablation.
Multi-scale readouts tie within noise on AP@0.3; the single-scale $\{3\}$-only
readout collapses and $\{6\}$-only trails on AP@0.5, so we keep the cheapest
two-scale configuration $\{3,6\}$.}
\label{fig:app-readout}
\end{figure}

Three ablation axes are isolated under the identical head and training recipe: the
gating rank $K$ (Figure~\ref{fig:app-rankk-sweep}), the readout
(Figure~\ref{fig:app-readout}), and the head components
(Figure~\ref{fig:app-readout-head}). The aggregation-operator ladder under the same
shared backbone is in Table~\ref{tab:operator}.

\paragraph{Rank $K$.}
AP rises monotonically from $K{=}1$ to $K{=}8$ (${+}4.4$ AP@0.3, ${+}9.4$ AP@0.5) and
saturates: $K{=}16$ adds nothing on AP@0.3 and $0.4$ points on AP@0.5 despite $31\%$
more parameters and $25\%$ higher latency. The saturation point coincides with the
operator spectrum measurement (Appendix~\ref{app:rankk}), where the learned kernel is
dominated by its leading components, so additional ranks fit noise. Note that the
strict-IoU column gains twice as much as the lenient one (${+}9.4$ vs.\ ${+}4.4$):
richer gating buys precise pocket delineation, not coarse localization.

\paragraph{Readout.}
The two-scale $\{3,6\}$ readout is robust (Figure~\ref{fig:app-readout}): every
multi-scale tap---$\{2,6\}$ ($0.616$/$0.462$), $\{2,4,6\}$ ($0.617$/$0.458$), and
all six blocks ($0.616$/$0.460$)---is within noise of the default on AP@0.3, while
the default keeps the best AP@0.5 at the lowest cost. Single-scale readouts expose
the division of labor: $\{6\}$-only matches AP@0.3 ($0.617$) but drops $1.0$ point
on AP@0.5, whereas $\{3\}$-only collapses to $0.542$/$0.350$---deep features suffice
for coarse localization, but precise delineation needs the shallow tap.

\paragraph{Contrastive denoising.}
Removing CDN costs $0.6$ AP@0.3 and $1.5$ AP@0.5 at identical latency. That the
stricter threshold again benefits more matches CDN's role: the contrastive negatives
teach the decoder to reject near-miss decoys, which is exactly what a high-IoU
evaluation penalizes.

Across all three groups, design choices move AP@0.5 about twice as much as AP@0.3---the
strict-IoU column is where content--geometry binding pays off.

\subsection{Qualitative analysis protocols (Figure~\ref{fig:gallery})}
\label{app:gallery}

\paragraph{Predictions and matching.}
All panels are rendered from a single dump of the released \methodpocket checkpoint
(UniSite-DS test split, 2{,}293 proteins; sequences longer than 800 residues excluded at
inference). Predicted masks are binarized at a mask-logit sigmoid of $0.5$; the query
score is the product of its classification probability and the mean mask probability
over the predicted region. Matching follows the AP evaluation exactly: per protein,
queries sorted by descending score greedily claim their best-IoU ground-truth pocket,
and a query is a true positive iff its best IoU reaches the threshold and that pocket is
not already covered; each ground-truth pocket is claimable at most once. IoU is computed
on residue sets (boolean intersection over union of the index sets). The UniSite-3D
comparison uses its official per-protein predictions on the same split, scored by the
identical matching code and ground truth.

\paragraph{(A) Error maps.}
Two in-distribution test proteins (Q57849, P00359), three columns (ground truth,
UniSite-3D, \methodpocket). For \methodpocket we show, among the top-5 queries, the
prediction with the highest IoU against the union of ground-truth pocket masks; for
UniSite-3D we show its highest-scoring prediction. Residues are colored by set
operations on this pair of masks: TP $=$ pred\,$\cap$\,GT (green), FP $=$
pred\,$\setminus$\,GT (blue), FN $=$ GT\,$\setminus$\,pred (red); the ground-truth
column shows the pocket in purple. The annotation under each prediction is the
residue-set IoU of the shown mask against the ground-truth union. All views are
ray-traced molecular surfaces rendered with PyMOL, oriented so that the pocket opening
faces the camera (viewing direction from the protein C$_\alpha$ centroid to the pocket
C$_\alpha$ centroid), with a 10\,\AA\ scale bar calibrated by two pseudoatoms.

\paragraph{(B) Pocket anatomy.}
The P00359 pocket with its bound NAD ligand (sticks; C yellow, O red, N blue, P orange),
the whole protein as a white semi-transparent surface with the ground-truth pocket
tinted purple. The close-up magnifies the dashed region---anchored on the ligand---with
the same TP/FP/FN error coloring as (A).

\paragraph{(C) Query constellation.}
All 50 decoder queries of a single forward pass on P00359, rendered as spheres on the
molecular surface. Each sphere sits at the C$_\alpha$ centroid of the query's predicted
mask (queries with an empty mask are not drawn); the radius is $1.2$\,\AA, enlarged to
$2.0$\,\AA\ when the query score exceeds $0.5$. Sphere color encodes the PPN pocketness
score on a linear grey-blue$\,\to\,$vermillion scale. The purple patch is the top-1
query's predicted mask.

\paragraph{(D) PR curves.}
Precision--recall curves are \emph{pooled}: the (score, TP/FP) records of all queries of
all test proteins are concatenated, sorted globally by descending score, and
accumulated; recall is relative to the total number of ground-truth pockets in the
split. AP is the area under the precision envelope (trapezoidal). Curves use all 50
queries per protein at IoU $0.3$ and $0.5$; their areas are the AP numbers of
Table~\ref{tab:pocket}. The UniSite-3D curve is produced by the same pooling pipeline
applied to its official predictions.

\paragraph{(E) Score histogram.}
Distribution of the top-5 query scores per protein over the same split (40 equal-width
bins on $[0,1]$), split into true and false positives at IoU $0.5$ by the matching rule
above. Note the different query budget from (D): the histogram reflects the
top-5 operating point, whereas the PR curves integrate over all 50 queries.

\section{Proofs}
\label{app:proofs}

\subsection{Proof of Proposition~\ref{prop:containment} (containment)}
\label{app:proof-containment}

\begin{proof}
\emph{Channel-gated $\subset$ matrix-gated.}
Take the diagonal matrix-valued kernel $\mW(\vdelta)=\mathrm{diag}(\mathbf{w}(\vdelta))$;
then $\mW(\vdelta)\vx=\mathbf{w}(\vdelta)\odot\vx$ is exactly class~\eqref{eq:channel}.

\emph{Scalar-gated $\subset$ channel-gated.}
Absorb the fixed linear map $\mV$ into the content features, $\tilde\vx_j=\mV\vx_j$, and
take the channel gate $\mathbf{w}(\vdelta)=\alpha(\vdelta)\,\mathbf{1}$ with all entries
tied. Then $\mathbf{w}(\vdelta)\odot\tilde\vx_j=\alpha(\vdelta)\,\mV\vx_j$ is exactly
class~\eqref{eq:scalar}.

\emph{Additive MP $\subset$ matrix-gated.}
Augment the content with a constant channel, $\tilde\vx=[\vx;1]\in\mathbb{R}^{C+1}$, and
define the geometry-affine kernel
$\mW(\vdelta)=[\,\mW\;|\;\mU\vdelta\,]=\mB_0+\sum_{m=1}^{d_g}\delta_m\,\mB_m$ with
$\mB_0=[\mW\;|\;0]$ and $\mB_m=[\,0\;|\;\mU\mathbf{e}_m\,]$. Then
$\mW(\vdelta_{ij})\tilde\vx_j=\mW\vx_j+\mU\vdelta_{ij}$ is exactly
class~\eqref{eq:additive}: additive MP is the matrix-gated layer whose kernel is affine
in $\vdelta$ (first-order in geometry, no learned geometry-dependent modulation beyond
the linear term).

\emph{Matrix-gated $\subset$ complete bilinear.}
Fix a basis $\{\mB_c\}$ of the matrix space containing the kernel's range and write
$\mW(\vdelta)=\sum_c g_c(\vdelta)\,\mB_c$ (any continuous kernel admits such an expansion;
for a universal $\phi$ whose coordinates include the coefficient functions
$g_1,\dots,g_K$, e.g.\ a Schauder basis of the continuous functions on the compact edge
domain, every continuous kernel is represented exactly). Define the linear map
$\mathcal{T}$ on $\mathbb{R}^{C}\otimes\mathbb{R}^{d_\phi}$ by
$\mathcal{T}(\vx\otimes\mathbf{e}_c)=\mB_c\vx$ on the basis tensors, extended linearly.
Then
\begin{equation*}
\mathcal{T}\big(\vx\otimes\phi(\vdelta)\big)
=\sum_c g_c(\vdelta)\,\mathcal{T}(\vx\otimes\mathbf{e}_c)
=\sum_c g_c(\vdelta)\,\mB_c\vx=\mW(\vdelta)\,\vx ,
\end{equation*}
which is class~\eqref{eq:matrix}.
\end{proof}

\subsection{Proof of Proposition~\ref{prop:binding} (binding blindness)}
\label{app:proof-binding}

\begin{proof}
For an additive layer \eqref{eq:additive}, both configurations give
\begin{equation*}
\vh_i^{A}=\mW\vx_1+\mW\vx_2+\mU\vdelta_1+\mU\vdelta_2=\vh_i^{B},
\end{equation*}
since the content sum and the geometry sum factorize and addition is commutative; no
parameter choice can make $\vh_i^{A}\neq\vh_i^{B}$.

For a matrix-gated layer, the output difference is
\begin{equation*}
\vh_i^{A}-\vh_i^{B}
=\mW(\vdelta_1)\vx_1+\mW(\vdelta_2)\vx_2-\mW(\vdelta_2)\vx_1-\mW(\vdelta_1)\vx_2
=\big(\mW(\vdelta_1)-\mW(\vdelta_2)\big)(\vx_1-\vx_2),
\end{equation*}
which is nonzero whenever
$\mW(\vdelta_1)(\vx_1-\vx_2)\neq\mW(\vdelta_2)(\vx_1-\vx_2)$. Such kernels exist: since
$\vdelta_1\neq\vdelta_2$, a continuous matrix-valued kernel may take $\mW(\vdelta_1)=I$
and $\mW(\vdelta_2)=0$ (and universal weight networks approximate it), giving
$\vh_i^{A}=\vx_1\neq\vx_2=\vh_i^{B}$.
\end{proof}

\subsection{Proof of Theorem~\ref{thm:ceiling} (the bilinear ceiling)}
\label{app:proof-ceiling}

\begin{proof}
Let the message $m(\vx_j,\vdelta_{ij})$ be second-order, i.e.\ bilinear in
$(\vx_j,\phi(\vdelta_{ij}))$. By the universal property of the tensor product, every
bilinear map on $\mathbb{R}^{C}\times\mathbb{R}^{d_\phi}$ factors uniquely through
$\vx_j\otimes\phi(\vdelta_{ij})$: there exists a linear $\mathcal{T}$ with
$m(\vx_j,\vdelta_{ij})=\mathcal{T}\big(\vx_j\otimes\phi(\vdelta_{ij})\big)$. Hence the
message depends on $(\vx_j,\vdelta_{ij})$ only through their outer product, which is
therefore a sufficient statistic of all second-order interaction. Conversely, every
linear $\mathcal{T}$ on $\mathbb{R}^{C}\otimes\mathbb{R}^{d_\phi}$ defines a message
$\mathcal{T}(\vx_j\otimes\phi(\vdelta_{ij}))$ that is bilinear in
$(\vx_j,\phi(\vdelta_{ij}))$, so class~\eqref{eq:bilinear} realizes \emph{all} such
messages. Neighborhood summation and layer stacking preserve the class by linearity, so
no composition of such layers exceeds this ceiling.
\end{proof}

\subsection{Exactness of the outer-product factorization}
\label{app:outerprod}

\begin{lemma}[Outer product realizes matrix gating]
\label{lem:outerprod}
Let $P_c:\mathbb{R}^{K\times C}\to\mathbb{R}^{C'}$ be linear, identified with the block
matrix $P_c=[\,P_c^{(1)}\,|\,\cdots\,|\,P_c^{(K)}\,]$,
$P_c^{(k)}\in\mathbb{R}^{C'\times C}$, acting on the row-major
flattening of $\mathbb{R}^{K\times C}$. Then for every $\mathbf{g}\in\mathbb{R}^{K}$ and
$\vh\in\mathbb{R}^{C}$,
\begin{equation}
\label{eq:outerprod-exact}
P_c\big(\mathbf{g}\otimes \vh\big)=\Big(\sum_{k=1}^{K} g_k\,P_c^{(k)}\Big)\vh .
\end{equation}
Consequently, with $\vh_j=\mW_v\vx_j$ and $\mB_k=P_c^{(k)}\mW_v$, the layer of
Eq.~\eqref{eq:hfconv} coincides exactly with the matrix-gated form of
Eq.~\eqref{eq:hfkernel} with kernel
$\mW(\vdelta,\Delta s)=\sum_{k=1}^{K}\tilde g_k(\vdelta,\Delta s)\,\mB_k$,
$\tilde g_k=\sigma\,g_k$.
\end{lemma}

\begin{proof}
The row-major flattening of $\mathbf{g}\otimes\vh$ is the concatenation of the blocks
$g_k\vh\in\mathbb{R}^{C}$, $k=1,\dots,K$. Block-wise multiplication gives
$P_c(\mathbf{g}\otimes\vh)=\sum_{k}P_c^{(k)}(g_k\vh)=\sum_{k}g_k(P_c^{(k)}\vh)
=\big(\sum_{k}g_kP_c^{(k)}\big)\vh$,
where the second equality is linearity of $P_c^{(k)}$ in the face of the scalar $g_k$.
Applying this edge-wise with $g_k=\sigma_{ij}\,g_k(\vdelta_{ij},\Delta s_{ij})$,
$\vh_j=\mW_v\vx_j$, and summing over $j\in\mathcal{N}(i)$ yields Eq.~\eqref{eq:hfkernel}
with $\mB_k=P_c^{(k)}\mW_v$.
\end{proof}

\subsection{A spectral tail bound for rank-$K$ gating}
\label{app:proof-rankk}

We quantify the error of restricting the edge-conditioned kernel to a rank-$K$ basis. Let
the edge descriptor $(\vdelta,\Delta s)$ range over its data-induced distribution and let
$\mW(\vdelta,\Delta s)\in\mathbb{R}^{C'\times C}$ denote the learned full operator family,
e.g.\ the operators induced by the $K{=}16$ model measured in
Appendix~\ref{app:rankk}. Write $\mathbf{w}=\mathrm{vec}(\mW)\in\mathbb{R}^{C'C}$ and let
the second-moment matrix $M=\mathbb{E}[\mathbf{w}\mathbf{w}^{\!\top}]\succeq 0$ have
eigenvalues $\lambda_1\ge\cdots\ge\lambda_{C'C}$ with orthonormal eigenvectors
$\mathbf{v}_c$; set $\mV_c=\mathrm{mat}(\mathbf{v}_c)\in\mathbb{R}^{C'\times C}$.

\begin{proposition}[Spectral tail bound]
\label{prop:spectral}
For every choice of $K$ basis operators
$\{\mB_c\}_{c=1}^{K}\subset\mathbb{R}^{C'\times C}$ and every choice of measurable
coefficient functions $g_c(\vdelta,\Delta s)$,
\begin{equation}
\label{eq:tailbound}
\mathbb{E}\,\Big\|\,\mW(\vdelta,\Delta s)-\sum_{c=1}^{K}g_c(\vdelta,\Delta s)\,\mB_c
\,\Big\|_F^{2}\;\ge\;\sum_{c>K}\lambda_c .
\end{equation}
Equality holds for the eigenbasis $\mB_c=\mV_c$ with the orthogonal-projection coefficients
$g_c(\vdelta,\Delta s)=\langle\mW(\vdelta,\Delta s),\mV_c\rangle_F$.
\end{proposition}

\begin{proof}
The best-approximation error depends on $\{\mB_c\}$ only through
$\mathrm{span}\{\mB_c\}$, so by Gram--Schmidt we may take the $\mB_c$ orthonormal with
respect to the Frobenius inner product. For a fixed basis the error is minimized pointwise
in $(\vdelta,\Delta s)$ by the orthogonal projection
$g_c=\langle\mW,\mB_c\rangle_F$, and Pythagoras gives, with
$\mathbf{b}_c=\mathrm{vec}(\mB_c)$,
\begin{equation*}
\mathbb{E}\,\Big\|\mW-\sum_{c\le K}g_c\mB_c\Big\|_F^{2}
=\mathbb{E}\,\|\mathbf{w}\|_2^{2}-\sum_{c=1}^{K}\mathbb{E}\,\langle\mathbf{w},
\mathbf{b}_c\rangle^{2}
=\mathrm{tr}(M)-\sum_{c=1}^{K}\mathbf{b}_c^{\!\top}M\,\mathbf{b}_c .
\end{equation*}
By Ky Fan's maximum principle,
$\max\sum_{c=1}^{K}\mathbf{b}_c^{\!\top}M\mathbf{b}_c=\sum_{c=1}^{K}\lambda_c$ over
orthonormal sets $\{\mathbf{b}_c\}$, attained by the top-$K$ eigenvectors. The minimum
error is therefore $\mathrm{tr}(M)-\sum_{c\le K}\lambda_c=\sum_{c>K}\lambda_c$, attained
by $\mB_c=\mV_c$.
\end{proof}

\begin{corollary}[Relative error of the deployed model]
\label{cor:relerr}
The relative operator error of the optimal rank-$K$ factorization is the spectral tail
fraction,
$\mathbb{E}\|\mW-\widehat\mW_K\|_F^{2}/\,\mathbb{E}\|\mW\|_F^{2}
=1-\big(\sum_{c\le K}\lambda_c\big)/\mathrm{tr}(M)$.
With the top-8 energy fractions measured in Figure~\ref{fig:app-spectrum}
($89.9$--$98.2\%$ per block), rank-$8$ gating incurs a relative operator error of at most
$10.1\%$ in every backbone block, and at most $8.2\%$ in blocks $1$--$3$.
\end{corollary}

\begin{corollary}[Error propagation through the layer]
\label{cor:propagate}
Let $\widehat\mW$ be any rank-$K$ approximation of $\mW$ and write
$\Delta\mW_{ij}=\mW(\vdelta_{ij},\Delta s_{ij})-\widehat\mW(\vdelta_{ij},\Delta s_{ij})$.
Then the outputs of Eq.~\eqref{eq:hfkernel} satisfy
\begin{equation*}
\|\vh_i-\widehat\vh_i\|_2^{2}\;\le\;
\Big(\sum_{j\in\mathcal{N}(i)}\|\Delta\mW_{ij}\|_F^{2}\Big)
\Big(\sum_{j\in\mathcal{N}(i)}\|\mW_v\vx_j\|_2^{2}\Big) .
\end{equation*}
\end{corollary}

\begin{proof}
$\vh_i-\widehat\vh_i=\sum_{j}\Delta\mW_{ij}\,\mW_v\vx_j$; apply the triangle inequality,
the bound $\|\Delta\mW_{ij}\mW_v\vx_j\|_2\le\|\Delta\mW_{ij}\|_F\,\|\mW_v\vx_j\|_2$, and
Cauchy--Schwarz.
\end{proof}

\begin{remark}[The tail bound is a floor, not the achieved error]
Eq.~\eqref{eq:tailbound} lower-bounds the error of \emph{every} rank-$K$ scheme whose
coefficients are arbitrary functions of the edge descriptor. \method's gates are further
constrained to the hypothesis class of the bucketed weight network
(Eq.~\eqref{eq:weightnet}), so the achieved error lies above the spectral tail; the
saturation of the $K$-sweep in Figure~\ref{fig:app-rankk-sweep} indicates that at $K{=}8$ the
residual error is dominated by this realizability gap rather than by the rank.
\end{remark}

\begin{remark}[The realizability gap vanishes with capacity]
\label{rem:universality}
The edge descriptor $(\vdelta,\Delta s)$ ranges over a compact domain, the bucket map is
piecewise constant, and within each bucket the weight network is an MLP with LeakyReLU
activations---a universal approximator of continuous functions. Any measurable
coefficient map $g_c(\vdelta,\Delta s)$ can therefore be approximated uniformly to
arbitrary precision as the hidden width grows, so the gap between the achieved rank-$K$
error and the spectral floor of Proposition~\ref{prop:spectral} is a capacity effect, not
a structural limit of the factorization.
\end{remark}

\begin{remark}[Sequence clamping]
The discretization $\mathrm{bucket}(\Delta s)$ is exact for $|\Delta s|<\ell/2$ (each
integer offset owns a bucket) and merges all pairs with $|\Delta s|\ge\ell/2$ into the two
boundary buckets. The smooth gate $\sigma_{ij}$ is decreasing in
$\bar s=|\Delta s_{ij}|/(\ell/2)$ and vanishes as $\bar d\,\bar s\to 1$, so the merged
pairs---those maximal in sequence distance---are precisely the ones whose contribution is
softly removed, which attenuates the clamping error in the layer output.
\end{remark}

\subsection{Operator spectrum measurement}
\label{app:rankk}

To substantiate Proposition~\ref{prop:rankk} we measure the spectrum of the learned full
operators of a $K{=}16$ \methodpocket model---deliberately over-parameterized relative to the
deployed $K{=}8$, so the measurement reveals how many gating directions the model actually
uses. For each block we collect the gate coefficients $g(\vdelta_{ij})\in\mathbb{R}^{K}$ over
all edges of 64 validation structures into the second-moment matrix
$G=\mathbb{E}[gg^\top]$; with the Gram matrix of the basis matrices,
$\mathrm{Gram}_{cc'}=\langle \mB_c,\mB_{c'}\rangle_F$, the operator energy is
$\mathbb{E}\|\mW\|_F^2=\mathrm{tr}(G\,\mathrm{Gram})$, and diagonalizing
$G^{1/2}\,\mathrm{Gram}\,G^{1/2}$ yields the operator spectrum. We report the fraction of
its total eigenvalue energy captured by the top-$8$ eigendirections
(Figure~\ref{fig:app-spectrum}).

\begin{figure}[!h]
\centering
\includegraphics[width=0.9\textwidth]{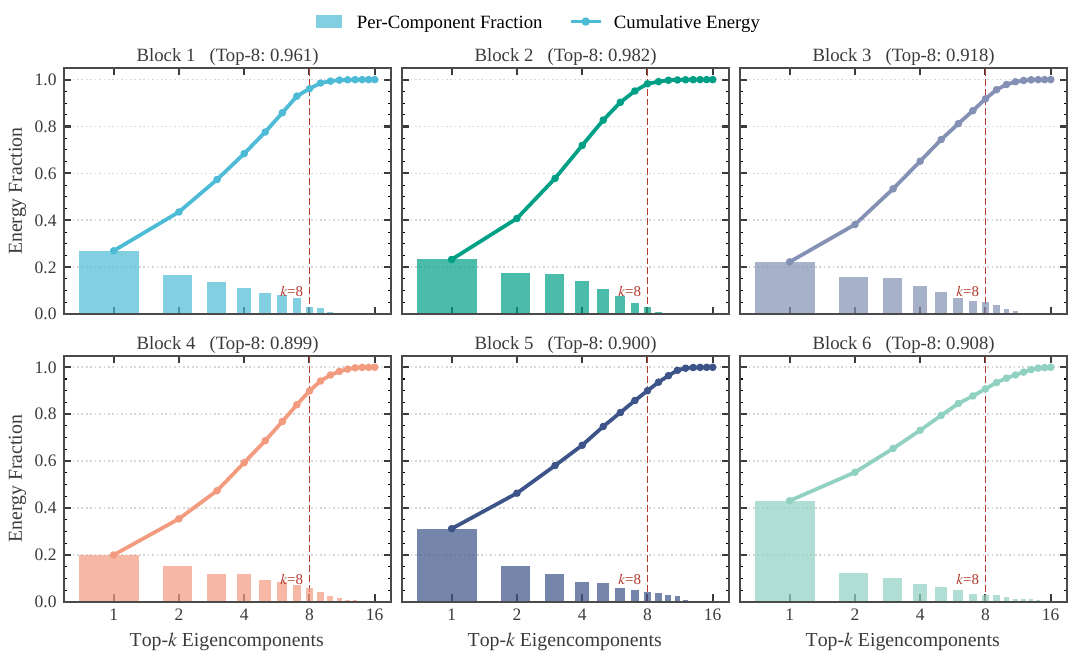}
\caption{Per-block operator spectrum of the $K{=}16$ \methodpocket model. Light bars
are the individual eigenvalue-energy fractions; the colored line is the cumulative
energy. The dashed line marks the deployed rank $K{=}8$; panel titles report the
top-8 fraction. Every block's spectrum decays rapidly, and deep blocks are no flatter
than shallow ones.}
\label{fig:app-spectrum}
\end{figure}

Consistent with this measurement, the $K$-sweep on pocket detection
(Figure~\ref{fig:app-rankk-sweep}) shows that doubling the basis from $K{=}8$ to $16$ (a
$31\%$ parameter cost) yields no AP@0.3 gain and only $+0.4$ AP@0.5
points---accuracy saturates where the spectrum does.

\end{document}